\documentclass{article}
\usepackage{conference,times}

\usepackage{amsmath,amsfonts,bm}

\def\eqref#1{equation~\ref{#1}}

\def\1{\bm{1}}

\DeclareMathAlphabet{\mathsfit}{\encodingdefault}{\sfdefault}{m}{sl}
\SetMathAlphabet{\mathsfit}{bold}{\encodingdefault}{\sfdefault}{bx}{n}

\usepackage[utf8]{inputenc} 
\usepackage[T1]{fontenc}    
\usepackage{hyperref}       
\usepackage{url}            
\usepackage{booktabs}       
\usepackage{amsfonts}       
\usepackage{nicefrac}       
\usepackage{microtype}      
\usepackage{xcolor}         
\usepackage{wrapfig} 
\usepackage{tikz}
\usetikzlibrary{positioning}
\usepackage{xspace}
\usepackage{amsthm}
\usepackage{subcaption}
\usepackage{amsmath}
\usepackage{graphicx}
\usepackage{wrapfig}
\usepackage{tabularx}
\usepackage[table]{xcolor}
\usepackage[dvipsnames]{xcolor}
\newtheorem{definition}{Definition}
\newtheorem{theorem}{Theorem}
\newtheorem{lemma}[definition]{Lemma}

\newtheorem{remark}[definition]{Remark}
\usepackage{tikz}
\usepackage{enumitem}
\usetikzlibrary{positioning, shapes.geometric, arrows.meta, fit, calc, backgrounds}
\usepackage{multirow}
\usepackage{times}  
\usepackage[normalem]{ulem}
\usepackage{xcolor}
\usepackage{pgfplots}
\usepackage{amssymb}
\usepackage{caption}

\definecolor{round1}{HTML}{f0e442}
\definecolor{round2}{HTML}{56B4E9}
\definecolor{round3}{HTML}{E69F00}

\definecolor{okb_lightorange}{HTML}{E69F00}
\definecolor{okb_lightblue}{HTML}{56B4E9}
\definecolor{okb_green}{HTML}{009E73}
\definecolor{okb_yellow}{HTML}{f0e442}
\definecolor{okb_darkblue}{HTML}{0072b2}
\definecolor{okb_orange}{HTML}{d55e00}
\definecolor{okb_magenta}{HTML}{cc79a7}

\newcommand*\circled[1]{\tikz[baseline=(char.base)]{
\node[shape=circle,draw,inner sep=1pt] (char) {#1};}}

\newsavebox{\pirshieldbox}
\newsavebox{\dbiconbox}

\sbox{\pirshieldbox}{
  \tikz[baseline=-0.6ex, scale=0.15]{
    \draw[fill=purple!70!black, draw=purple!70!black, line join=round] (0,1) -- (1,1) .. controls (1,0) and (0.5,-1) .. (0,-1.5) .. controls (-0.5,-1) and (-1,0) .. (-1,1) -- cycle;
    \draw[fill=white, draw=white, line join=round] (0,0.5) -- (0.5,0.5) .. controls (0.5,0) and (0.25,-0.5) .. (0,-0.75) .. controls (-0.25,-0.5) and (-0.5,0) .. (-0.5,0.5) -- cycle;
  }
}

\sbox{\dbiconbox}{
  \tikz[baseline=-0.5ex, scale=0.2]{
    \draw[fill=white, thick] (0, 1.5) ellipse (1 and 0.4);
    \draw[fill=white, thick] (-1, 0) -- (-1, 1.5) arc (180:360:1 and 0.4) -- (1, 0) arc (360:180:1 and 0.4);
    \draw[fill=white, thick] (-1, -1.5) -- (-1, 0) arc (180:360:1 and 0.4) -- (1, -1.5) arc (360:180:1 and 0.4);
    \draw[thick] (-1, 1.5) arc (180:360:1 and 0.4);
    \draw[thick] (-1, 0) arc (180:360:1 and 0.4);
  }
}

\newcommand{\bin}{\texttt{PILLAR-Bin}\xspace}
\newcommand{\tree}{\texttt{PILLAR-Tree}\xspace}

\newcommand{\iscore}{\textsf{score}\xspace}

\newcommand{\doc}{D}

\newcommand{\corpus}{\mathcal{C}}
\newcommand{\answer}{\mathcal{D}^k}

\newcommand{\vocab}{\mathcal{V}}

\newcommand{\client}{\texttt{clnt}\xspace}
\newcommand{\server}{\texttt{srv}\xspace}
\newcommand{\query}{Q}
\newcommand{\embdim}{d}
\newcommand{\analyze}{\mathsf{Analyze}}
\newcommand{\fpprag}{\mathsf{PPRAG}}
\newcommand{\emb}{\mathsf{Emb}}
\newcommand{\proto}{\textsf{PILLAR}\xspace}

\usetikzlibrary{arrows.meta, shapes, positioning, calc, backgrounds}
\usetikzlibrary{positioning, arrows.meta, shapes.geometric, shapes.symbols,
  fit, calc, decorations.pathreplacing, backgrounds,
  shadows.blur, patterns, fadings}
\usepgfplotslibrary{groupplots}

\definecolor{clientBlue}{HTML}{4A90D9}
\definecolor{agentOrange}{HTML}{E07B30}
\definecolor{llmPurple}{HTML}{7B61B0}
\definecolor{toolGreen}{HTML}{3A9D6A}
\definecolor{docYellow}{HTML}{D4A843}
\definecolor{lockRed}{HTML}{C0392B}
\definecolor{textDark}{HTML}{2C3E50}
\definecolor{arrowDark}{HTML}{444444}
\definecolor{loopRed}{HTML}{C0392B}
\definecolor{privacyGreen}{HTML}{27AE60}
\definecolor{localBlue}{HTML}{3B7DD8}

\newcommand{\database}{\textsf{DB}}
\newcommand{\ind}{\textsf{id}}
\newcommand{\macroname}{\texttt{PILLAR}\xspace}
\newcommand{\pcost}{\texttt{PIRCost}\xspace}

\title{PILLAR: Private Inverted-index Lexical Lookup for Augmented Retrieval}

\newcommand{\authw}[1]{\makebox[0.5\textwidth][l]{#1}}

\author{Truong Son Nguyen \thanks{Equal contribution} \\
\authw{Department of Computer Science} \\
Arizona State University \\
Tempe, AZ 85281, USA \\
\texttt{snguye63@asu.edu} \\
\And
Daniel Blackley \footnotemark[1]  \\
\authw{Department of Computer Science} \\
George Mason University \\
Fairfax, VA 22030, USA \\
\texttt{dblackle@gmu.edu} \\
\AND
Ni Trieu \\
\authw{Department of Computer Science} \\
Arizona State University \\
Tempe, AZ 85281, USA \\
\texttt{nitrieu@asu.edu} \\
\And
Evgenios M. Kornaropoulos \\
\authw{Department of Computer Science} \\
George Mason University \\
Fairfax, VA 22030, USA \\
\texttt{evgenios@gmu.edu} \\
}

\begin{document}

\maketitle

\begin{abstract}
Retrieval-augmented generation (RAG) hands the user's query to whoever hosts the corpus. We propose \texttt{PILLAR}, a Privacy-Preserving RAG (PPRAG) system based on Private Information Retrieval (PIR) in which a client utilizes the $k$ documents
most similar to their query from a server-held and publicly known corpus to respond to their query, while the server learns nothing about the query, either its terms or its access pattern. Prior PPRAG constructions rely on dense
retrieval alone, translating approximate nearest-neighbor search into many query-dependent rounds of PIR, and pay for it in both latency and retrieval quality. \texttt{PILLAR} instead performs \emph{private hybrid retrieval} in two stages. A sparse stage issues a small, fixed number of PIR queries against a carefully designed index of precomputed BM25 scores, filtering the corpus down to candidates that share terms with the query without the server ever seeing which terms these are. A dense stage then fetches only those
candidates' \emph{document embeddings} and re-ranks them locally, avoiding the many costly PIR queries that private dense retrieval typically requires. We instantiate \texttt{PILLAR} with two protocols that trade latency against retrieval
quality, each built on a different private rendering of lexical search. \texttt{PILLAR-Bin} bins posting lists into a hash table and is a single-round design that achieves lower latency than state-of-the-art private retrieval schemes. \texttt{PILLAR-Tree} turns block-max pruning into an oblivious tree traversal combined with cuckoo hash tables and
achieves the highest retrieval quality at lower latency than state-of-the-art schemes. 

\end{abstract}

\section{Introduction}

Retrieval-augmented generation (RAG) is now the standard way to ground a Large Language Model (LLM) in knowledge that is not in its parameters~\citep{lewis2020rag,gao2024ragsurvey}, and it is the mechanism major search engines rely on for AI summaries.
A RAG pipeline consists of a client that issues a query, a server that hosts a document corpus, and an LLM that generates an answer to the client's query by retrieving documents from the server to augment its generated response.
This poses a severe privacy risk, the query is handed, verbatim, to the server. 
Exposing a client's query carries real risk; search engine logs can trivially be used to identify clients~\citep{barbaro2006aol}, clients disclose sensitive personal details to LLMs~\citep{mireshghallah2024trustnobot}, and current models infer identifiable information from very little text~\citep{staab2024beyond}.

\textbf{Privacy-Preserving RAG.}
A privacy-preserving RAG (PPRAG) protocol answers a client's query over a public corpus such that the server hosting the corpus learns nothing about the query. 
Due to the public nature of the corpus, such protocols are typically built on Private Information Retrieval (PIR)~\citep{chor1995pir}, a cryptographic primitive that lets a client fetch a non-encrypted item from a server without revealing which item was fetched.
Existing PPRAG work falls into two categories, \emph{dense} and \emph{sparse} retrieval, with the majority focusing on the former.
\begin{wrapfigure}[21]{r}{0.35\linewidth}
    \centering
    \includegraphics[width=\linewidth]{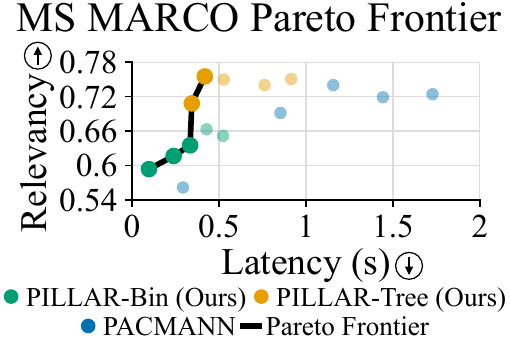}
    \caption{Pareto between  
    a state-of-the-art private dense retrieval method PACMANN, and
    our private hybrid retrieval methods (\bin \& \tree). In MS MARCO, our \tree (in orange) dominates across all quality metrics, while our \bin (in green) dominates in speed/latency.}
    \label{fig:best_grid}
\end{wrapfigure}
Among dense retrieval protocols, Tiptoe~\citep{tiptoe} privately searches over clustered embeddings, Compass~\citep{compass} traverses an HNSW graph using Oblivious RAM (ORAM), and PACMANN~\citep{pacmann}, the current state-of-the-art (SOTA), has the client traverse an approximate nearest neighbor (ANN) graph.
Sparse retrieval, by contrast, is represented by Coeus~\citep{coeus}, which scores the query against a term-frequency matrix. 

\textbf{Limitations of Existing PPRAG.}
Existing PPRAG protocols face two limitations: First, ANN traversal is adaptive; each hop depends on the query, and the walk terminates once no closer neighbors are found.
Hiding the access pattern therefore requires a fixed, query-independent schedule. 
On MS MARCO, PACMANN requires $20$ round trips per query yet achieves only $0.266$ MRR@10, compared to roughly $0.31$ for a non-private ANN baseline over identical embeddings~\cite{pacmann}. 
Second, existing protocols support only a single retrieval architecture, whereas \emph{real RAG deployments} commonly use hybrid architectures, with lexical scoring followed by comparing learned embedding similarities (See Appendix~\ref{app:hybrid_retrieval} for more details). 
The two architectures fail in complementary ways: learned embeddings encode word meanings using their training data but fail when a query uses the same word in a different context, while lexical scoring does not consider semantically relevant words~\citep{ma2021hybrid,thakur2021beir,bruch2023fusion}. 
Used together, the methods mitigate each other's weaknesses.
This points to a gap in PPRAG literature:

\begin{quote}
\emph{No PPRAG uses both sparse and dense retrieval architectures, and the single-architecture alternatives pay for it in number of interactions and in answer quality.}
\end{quote}

\textbf{Contributions.} 
We present $\proto$, a new \emph{hybrid-retrieval}\footnote{In standard literature, a ``hybrid-retrieval RAG protocol'' refers to a protocol that combines the scores from both dense and sparse retrieval; here hybrid-retrieval refers strictly to the retrieval methods, not the final score} PPRAG architecture that retains the strengths while limiting the weaknesses of previous dense or sparse PPRAG protocols.
\proto achieves a "best-of-both-worlds" design, and we demonstrate its effectiveness with two new PPRAG protocols: \bin and \tree.
\bin uses lexical scoring functions to resolve a query in a \emph{single} round, achieving low latency.
\tree adapts block-max BM25~\cite{broder2003wand, ding2011bmw}  to perform multiple rounds of increasingly accurate retrievals, achieving unmatched quality compared to SOTA.
We evaluate both the quality of the answer and latency for \bin, \tree, and the SOTA PACMANN, and provide an extensive evaluation in Section~\ref{sec:evaluation} and Appendix~\ref{app:tree-experiment}.  
An indicative result is in Figure~\ref{fig:best_grid}, which shows the  \emph{Pareto frontier} over the Top-5 configurations for each method. Specifically, it captures the trade-offs between the latency required to retrieve documents for a query (X-axis) and the Relevancy of an LLM-generated answer to that query, augmented with those documents (Y-axis).
The \emph{only configurations} on the frontier belong to either \bin or \tree, never once is PACMANN better in either latency or Relevancy.
In summary, our core contributions are as follows:
\begin{itemize}[leftmargin=*]
    \item \textbf{Hybrid Retrieval Under PIR.} We give the first PPRAG protocols that combine lexical and semantic scoring, and prove that both \bin and \tree hide a client's query from a semi-honest server in the PIR-hybrid model.
    \item \textbf{Comprehensive Evaluation.} We benchmark both protocols against the SOTA, PACMANN, and plot the best configurations of all three protocols using a \emph{Pareto frontier} in Section~\ref{sec:comparison}, finding that \proto accounts for $93\%$ of the configurations on the frontier.
    We test on two standard RAG datasets, MS MARCO~\cite{bajaj2018msmarco} and SciFact~\cite{wadden-etal-2020-fact_scifact}, reporting retrieval quality, communication cost, and latency. All code for \bin is available at \url{https://github.com/dkblackley/bi...}
     Unlike PACMANN, which evaluates retrieval alone, we also evaluate the end-to-end answers generated from the retrieved documents, using the Retrieval Augmented Generation Assessment metrics \emph{faithfulness} and \emph{answer relevancy}~\cite{es2025ragasautomatedevaluationretrieval}.
    \item \textbf{A Low-Latency PPRAG Protocol.} The fastest PACMANN configuration needs $0.293$\,s, while \bin retrieves documents in as little as $0.098$\,s on MS MARCO, with comparable quality.
    \item \textbf{A High-Quality PPRAG Protocol.} The highest quality documents across the best configurations are retrieved by \tree, with answer Relevancy $0.75$. The PACMANN configuration that finds the best documents achieves $0.72$ answer relevancy but requires over $2\times$ more latency.
\end{itemize}

\section{Background \& Preliminaries}
\label{sec:background_notation}
\textbf{\textbf{Notation.}} A \emph{corpus} $\corpus = \{D_1,\dots,D_n\}$ is an ordered set of $n$ documents.
The index $i$ of a document  $D_i \in \corpus$ is also referred to as a document's \emph{identifier}.
An \emph{analyzer} $\analyze$ maps raw text to a finite sequence of tokens called \emph{terms}. 
The \emph{vocabulary} $\vocab = \bigcup_{D \in \corpus} \analyze(D)$ is the set of all unique terms in a corpus. 
An \emph{embedding function} $\emb$ maps raw text to a $d$-dimensional dense vector.  
With the term \emph{document embedding} $e_D =\emb(D) \in \mathbb{R}^{\embdim}$, we refer to a document's dense vector representation.
A \emph{query} $\query$ is raw text supplied by a client, so that $\analyze(\query) = (w_1,\dots, w_{|Q|})$ are its \emph{query terms} and $e_\query = \emb(\query) \in \mathbb{R}^{\embdim}$ its \emph{query embedding}.
We sometimes abuse the notation $\query$ to denote both a sequence of terms and the index of documents to be returned through PIR.
We make this distinction clear by referring to the latter case as the \emph{PIR query}.

\textbf{Private Information Retrieval (PIR).} There are multiple variants of PIR~\cite{chor1995pir, 1997_kushilevitz_replication_not_needed_single_pir, Amos_2000_preprocessing_PIR_Single} but we focus on Single-server PIR with preprocessing designs~\cite{zhou2024PIANO, Gibbs_2020_Optimal_PIR_Preprocessing, Amos_2000_preprocessing_PIR_Single}. 
These PIR protocols involve interactions between a stateful client $\client$ and a server $\server$, consisting of a ($i$) preprocessing phase followed by a ($ii$) query phase.  
In the PIR preprocessing phase, $\server$'s input is a publicly known $\corpus$, which it sends to $\client$, who initializes a private state.
During the query phase, $\client$'s input is an index $0 \le i \le |\corpus|$ (w.r.t. to $\corpus$) and constructs a PIR query $\query$ to retrieve the requested document $D_i \in \corpus$. $\client$ sends $\query$ to $\server$ and $\client$ uses the responses from $\server$ and his internal state to output $D_i$.
A PIR protocol is \emph{secure} if $\server$ does not discover anything about the client's input for the requested index $i$.
Consistent with prior work~\cite{tiptoe, pacmann}, we consider a semi-honest $\server$, where $\server$ follows the protocol but tries to recover $i$.

\textbf{Privacy-Preserving Retrieval Augmented Generation ($\fpprag$) Functionality.} 
A standard PIR protocol is incompatible with PPRAG because a $\client$ takes query terms $\query$ as input rather than an index $i$, and outputs the documents most similar to $\query$ rather than the single document at index $i$, where similarity is measured by a function $\iscore(\cdot,\cdot)$ mapping a (query, document) pair to a score under some metric. 
For the cryptographic formalism, we consider the ideal two-party $\fpprag$ functionality, where the server and client both have input, but only the client has output, defined by $\fpprag(\corpus, \query) \;=\; \bigl(\,\bot,  \answer \bigr)$. In words: $\server$ takes in a public corpus $\corpus$ and outputs nothing (denoted $\bot$),  $\client$ takes as input a private $\query$ and outputs $k$ documents $\answer \subseteq \corpus$ such that, for every $D \in \answer$, and every $D' \in \corpus \setminus \answer$ we have that: 
$ \iscore(\query, D) \;\ge\; \iscore(\query, D')$.
The new protocol $\proto$ in this work realizes this ideal functionality of $\fpprag$.

\section{Challenges of PIR under RAG Architecture}

This section discusses sparse and dense retrieval architectures for RAG with a public corpus, with particular attention to the challenges that arise when scaling each to a PIR-based \emph{private analog}.

\label{sec:pir_challenges_sec_3}

\subsection{Sparse Retrieval and Privacy.} 
\label{sec:sparse_retrieval_and_privacy}

At a high level, \emph{sparse retrieval} scores each document (via lexical scoring functions) by the query terms appearing in it verbatim, weighting rare terms more heavily than common ones, and returns the $k$ highest-scoring documents.
Some well-known lexical scoring functions are TF-IDF~\cite{robertson2009bm25}, Okapi BM25~\cite{robertson2009bm25}, and SPLADE ~\cite{splade}.
Specifically, each score computed for a term $w \in \vocab$ in the document $D$ is a function of the frequency of $w$ in $D$.
To facilitate this, each document is modeled as a vector $v$ over $\vocab$ whose entry for a term is nonzero only if the document contains that $\vocab$ term (hence \emph{sparse}, since a document uses only a small fraction of the vocabulary). 
Scoring a \emph{single} $D$ against $\query$ thus amounts to looking up $v_D$'s entries at the coordinates named by the query's terms and summing them, i.e., a small set of index lookups on vector $v_D$, where the indices are exactly the sensitive input/query. 
The following exposition describes techniques without privacy in mind (thus, the indices are visible to the server).

\noindent\textbf{Inverted Indexes.} 
An end-to-end PPRAG query under sparse retrieval requires $\server$ to score $\query$ against \emph{every} $\doc \in \corpus$ and return only the $k$ highest-lexical-scoring documents.
Computing the lexical score for every document is wasteful since many documents typically won't contain any terms in the query and will always score $0$.
A common way to filter out these documents is to organize documents using an abstract data type called an \emph{inverted index}. 
Inverted indexes are dictionaries that map each term $w \in \vocab$ to the set of documents in which $w$ occurs, called the \emph{posting list}. 
Retrieval protocols built on an inverted index avoid this exhaustive scan: 
$\server$ uses the inverted index to fetch the posting list for each $w \in \query$ and scores $\query$ against \emph{only the relevant documents} that reside in the posting lists. 
We discuss two methods for instantiating an inverted index in this context: bin-based and block-based.

\noindent\textbf{Bin-based BM25.} An inverted index can be instantiated as a \emph{hash table}, where each hash table bin stores the posting list for each hashed term.
We refer to this simple approach as \emph{bin-based BM25.}
At query time, $\server$ hashes each $w \in \query$ to a bin, fetches the posting lists stored there, scores the documents they contain, and returns the $k$ highest-scoring ones to $\client$. 
The \textbf{challenges} of applying PIR on top of a bin-based BM25 are: 
\circled{1} In non-privacy-preserving bin-based BM25, the \client receives the top-$k$ documents that \server locally filtered across posting lists; however, in a PPRAG rendition, the server does not see $\query$, thus, it cannot perform this task. 
\circled{2} In non-privacy-preserving bin-based BM25, to minimize collisions we need a large hash table; however, PIR computation complexity scales superlinearly with the size of the hash table. %
This introduces a trade-off: a large hash table yields fewer posting lists (i.e., less filtering) but increases PIR complexity, while a small hash table yields more posting lists (i.e., more filtering) with faster PIR complexity.

\noindent\textbf{Block-based BM25~\cite{ding2011bmw}.}
In a typical RAG deployment, $\client$ sends $\query$ to $\server$ and receives only the $k$ most similar documents in $\corpus$ w.r.t.\ $\query$.
To speed up this search, $\server$ can \emph{efficiently dismiss} documents that cannot reach the top-$k$.
We now discuss one such strategy built on the inverted index, \emph{Dynamic Pruning}~\cite{TURTLE1995831_MAXSCORE, broder2003wand}.
The underlying intuition is as follows. 
Recall that a document's $D$ score is the sum of $D$'s per-term contributions over the query terms, so for a ``batch'' $\mathcal{B} \subseteq \corpus$ of documents, taking the single largest contribution of each $w \in \query$ \emph{across all} documents in $\mathcal{B}$ and summing\footnote{Note that no document in $\mathcal{B}$ need actually attain this bound, since the per-term maxima may come from different documents, e.g., $D_3$ for $q_1$, say, and $D_9$ for $q_2$.} these maxima provides an upper bound for the exact score of \emph{every} document in $\mathcal{B}$. 
If this over-approximation falls below the current top-$k$ exact score, the entire batch can be skipped without scoring any of its documents individually. 
Dynamic Pruning realizes this idea as follows:
At initialization, $\server$ partitions all documents into \emph{blocks} where each one contains a fixed number of documents. For each block $\mathcal{B}$ and term $w$, $\server$ precomputes $\sigma_w(\mathcal{B})$, the largest score between $w$ and any document in $\mathcal{B}$ (thus, $|\vocab|$ scores per block).
At query time, $\server$ retrieves the blocks for each $w \in \query$ and scores the first $k$ documents arbitrarily to establish a baseline threshold. It then proceeds block by block: since no document in $\mathcal{B}$ can score higher than $\sum_{w \in \query} \sigma_w(\mathcal{B})$, a bound below the $k$th best exact score seen so far rules out the whole block, which is skipped, unscored. Blocks clearing the threshold are scored in full and the running top-$k$ updated. Once all blocks are traversed, $\server$ returns the final top-$k$; we refer to this as \emph{Block-based BM25}.

Applying PIR to this procedure is challenging:
   \circled{1} The algorithm attempts to find the top scoring documents but, given the lack of organization among  blocks, the server needs to linearly scan all blocks to identify the right ones. 
\circled{2} 
In non-private Block-based BM25, $\server$ tracks the $k$th best score seen so far, a threshold that is \emph{query-dependent}. 
Under PPRAG, $\server$ never sees (sensitive)  $\query$ and therefore \emph{cannot compute} this threshold, let alone decide which blocks it prunes.

\subsection{Dense Retrieval and Privacy.}
Whereas sparse retrieval matches on exact term overlap, dense retrieval matches on \emph{meaning}, so a document can rank highly without sharing a single word with the query.
Dense retrieval maps documents and queries into a shared vector space in which semantically similar texts lie close together~\cite{karpukhin-etal-2020-dense-vs-sparse-comparison, murphy2013machine}.
A popular measure of closeness is cosine similarity: given a query embedding $\mathbf{e}_\query \in \mathbb{R}^\embdim$ and a document embedding $\mathbf{e}_\doc \in \mathbb{R}^\embdim$, we compute $\cos(\mathbf{e}_\query, \mathbf{e}_\doc) = (\mathbf{e}_\query \cdot \mathbf{e}_\doc) / (\|\mathbf{e}_\query\| \|\mathbf{e}_\doc\|)$.

\noindent\textbf{Approximate Nearest Neighbor Search.} In a typical dense RAG deployment, $\client$ sends its query embedding $\mathbf{e}_\query$ to $\server$, which computes the cosine similarity against every document embedding and returns the $k$ closest. This is exactly $k$-Nearest Neighbor (NN) search~\cite{1998_Har_peled_ANN_approx_STOC}, at a cost of $O(|\corpus|\embdim)$ per query. The common alternative is \emph{Approximate} Nearest Neighbor (ANN) search~\cite{karpukhin-etal-2020-dense-vs-sparse-comparison, pacmann}, which in one popular form~\cite{pacmann} represents $\corpus$ as a graph\footnote{This converts a high-dimensional geometric search into a combinatorial one, with the geometry encoded into the edge set rather than evaluated at query time.} whose nodes are documents and whose edges join embeddings of high cosine similarity, reducing retrieval to a traversal toward increasingly similar documents rather than a full scan. 
Dense retrieval poses its own challenges under PIR: 
\circled{1} The graph traversal is \emph{query-dependent}: from an arbitrary starting node, $\server$ scores $\mathbf{e}_\query$ against the current node's neighbors and moves to the closest one, so which node is accessed at each step is determined by the query. Revealing this access pattern would \emph{leak information} about $\mathbf{e}_\query$.
\circled{2} The number of traversal steps also varies with the query, since the search halts once it stops finding closer neighbors. This count is itself a function of $\mathbf{e}_\query$, so revealing it would \emph{leak information} about the query. A PPRAG-friendly ANN must therefore fix the step count in advance, trading retrieval accuracy against cost on every query.

\textbf{Properties of a Hybrid Retrieval Protocol.} Sparse and dense retrieval succeed on complementary dimensions. Sparse retrieval pinpoints domain-specific terms that embeddings often fail to represent, while dense retrieval captures semantic relationships that exact lexical matching misses. A hybrid PIR protocol should therefore combine both, inheriting
\circled{A} the low, \emph{fixed round complexity} of sparse search,
\circled{B} sparse retrieval's ability to \emph{quickly dismiss documents} with no lexical overlap with the query, and
\circled{C} the \emph{semantic accuracy} of dense embeddings.

\section{PILLAR: \emph{PIR-Friendly} Hybrid Retrieval}
\label{sec:PILLAR_hybrid_retrieval}

This section presents our protocols $\bin$ and $\tree$, both of which realize $\fpprag$. Their design follows directly from the three properties a hybrid architecture should have, stated at the end of Section~\ref{sec:pir_challenges_sec_3}. Regarding
\circled{A}, every adaptive component is removed. The number of rounds and the size of every message depend only on publicly known parameters, never on $\query$, so the communication pattern is fixed.
For \circled{B}, the crude part of sparse retrieval runs first on the server, and the refined filtering stage on the client. Lexical scoring dismisses the bulk of $\corpus$ cheaply, leaving a small candidate set for the expensive stage.
\circled{C} Semantic re-ranking is moved completely to $\client$. After the sparse stage, $\client$ fetches the candidates' embeddings via PIR and ranks them locally, recovering the accuracy of dense retrieval without $\server$ ever learning $\query$. 
We provide the full protocols for both \bin \& \tree in Appendix \ref{app:protocol_details} and the security analysis in Appendix~\ref{app:security}.

\subsection{A Simple Bin-Based BM25}
\label{sec:bm25bin}

We discuss solutions for the two challenges identified for Bin-based BM25 in Section~\ref{sec:sparse_retrieval_and_privacy}.%

\noindent\textbf{Solving Bin-Based BM25 Challenges.} \circled{1} Recall the first problem for Bin-Based BM25 was that, in non-privacy-preserving Bin-Based BM25, \server would retrieve posting lists for each term $w \in \query$ and score the documents in the posting lists, returning just the top-$k$ highest-scoring documents. 
Scoring documents requires access to $\query$, which \server does not see in a PPRAG protocol.
To solve this, \server will only return a fixed number of $t$ documents per query term to \client, who can locally re-rank without revealing $\query$. 
To ensure the $t$ returned documents are still relevant, we return documents from a posting list that is sorted by frequency of the hashed term $w \in \vocab$ used to retrieve the posting list, and the $t$ highest-frequency documents are returned.
Empirical approaches for choosing $t$ are explored in Section~\ref{sec:evaluation}.
\circled{2} Recall the trade-off: a large hash table keeps collisions rare but drives up PIR complexity, while a small one requires less PIR operations but conflates unrelated terms into the same bin. 
$\bin$ uses only $R \ll |\corpus|$ rows and maps each term $w$ to row $H(w) \pmod R$ under a public hash function $H$. 
Thus, $\client$ can compute which row to request, and PIR hides that index from $\server$. 
$R$ is also smaller than the vocabulary, so many terms share a row, and the postings lists of all colliding terms are merged into it.
Rows are truncated at $t$ length, so, crucially,  the order of merging lists determines which documents survive truncation. 
Naively appending lists would preserve the first list in full and cut off the last list entirely. 
$\bin$ therefore \emph{interleaves} lists in a round-robin fashion, placing the $j$th document of the $i$th list at position $(j-1)m + i$ when $m$ terms collide. Reading the first $tm$ positions then yields the top $t$ documents from each colliding list. 
Finally, the $\server$ omits duplicates when a document is contained in two colliding lists.

\noindent\textbf{Final Design Description.}
At setup, $\server$ computes the embedding $\mathbf{e}_\doc \in \mathbb{R}^\embdim$ of every $D \in \corpus$ and stores embeddings in place of documents, so that $\client$ can re-rank locally once retrieval completes, i.e., the dense retrieval part of the hybrid design. It then builds the hash table as in non-private bin-based BM25 and applies the two designs above, i.e., hashing over $R$ rows and interleaving colliding postings lists round-robin. The result is a PIR database of $R$ rows, each holding exactly $t$ embeddings (after truncation) and hence $t \cdot \embdim$ field elements\footnote{Thus, all rows are indistinguishable in size regardless of how many terms collide in them.}. 
Since a query contains several terms, retrieving one row per term with independent PIR queries would cost $|\query|$ separate protocol executions; we instead use \emph{batched PIR}, which packs multiple indices into a single query at sublinear amortized cost per index~\cite{pacmann}. The batch size is fixed to the padding bound of \circled{3} rather than to $|\query|$, so the query length is identical across executions and leaks nothing about how many terms the client used. 
$\client$ issues one batched PIR query, decodes the $t$ embeddings in each returned row, and aggregates them into a candidate pool $\mathcal{D}_{\mathrm{cand}}$ of at most $t$ times the batch size, discarding duplicates that arise when a document appears in several retrieved rows. 
The entire protocol is a \emph{single round}, and its message sizes depend on $R$, $t$, $\embdim$, and the batch bound.

\subsection{Tree-based BM25 (\tree)}
\label{sec:bm25tree}
This section discusses how to solve the previously identified challenges for block-based BM25  in Section~\ref{sec:sparse_retrieval_and_privacy} and provides the design principle for $\tree$.

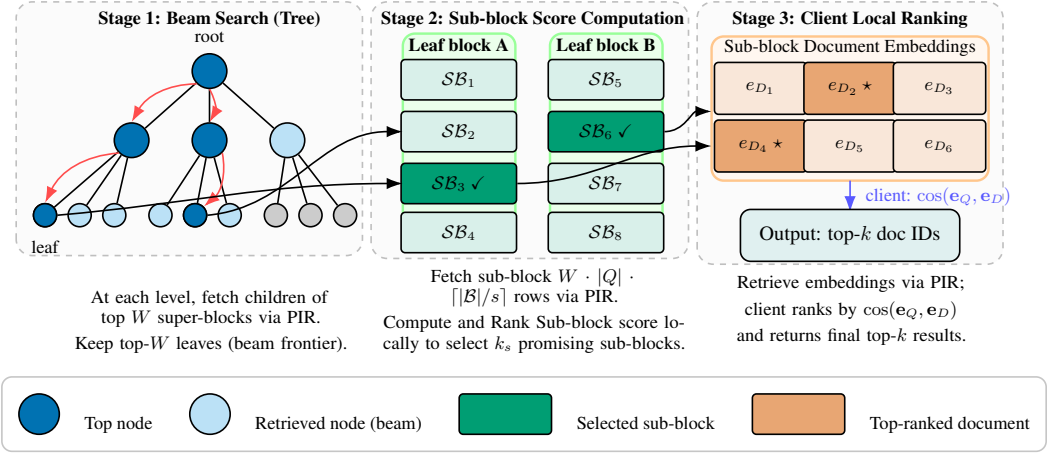
\begin{figure}
\resizebox{\textwidth}{!}{
\begin{tikzpicture}[
  treeNode/.style={draw, circle, minimum size=6mm, inner sep=0pt, thick, fill=gray!20},
  treeNodeChoice/.style={treeNode, fill=okb_darkblue},
  selNode/.style={treeNode, fill=okb_lightblue!40},
  leaf/.style={treeNode, minimum size=4mm, fill=gray!40},
  selLeaf/.style={treeNode,minimum size=4mm,  fill=okb_darkblue, thick},
  compLeaf/.style={treeNode,minimum size=4mm,  fill=okb_lightblue!40, thick},
  subRow/.style={draw, rectangle, rounded corners=2pt,
    minimum width=20mm, minimum height=7mm,
    font=\small, align=center, fill=okb_green!15, thick},
  selRow/.style={subRow, fill=okb_green},
  docBox/.style={draw, rectangle, rounded corners=2pt,
    minimum width=16mm, minimum height=9mm,
    font=\small, align=center, fill=okb_orange!15, thick},
  selDoc/.style={docBox, fill=okb_orange!55},
  arrow/.style={-{Latex[length=2.8mm,width=2mm]}, thick},
  redarrow/.style={arrow, red!70},
  bluearrow/.style={arrow, blue!65},
  panelBorder/.style={draw=gray!50, dashed, rounded corners=8pt, fill=gray!3, thick},
]

% ================================================================
%  STAGE 1 — Beam Search Tree    (panel y: 4.6..8.6)
% ================================================================
\begin{scope}[on background layer]
  \fill[panelBorder] (-1.5, 4.2) rectangle (4.5, 8.6);
\end{scope}
\node[font=\small\bfseries] at (1.85, 8.3) {Stage 1: Beam Search (Tree)};

\node[treeNodeChoice] (root) at (1.85, 7.4) {};
\node[font=\small, above=2pt of root] {root};

\node[treeNodeChoice] (na) at (0.5, 6.2) {};
\node[treeNodeChoice] (nb) at (1.85, 6.2) {};
\node[selNode] (nc) at (3.2, 6.2) {};

\node[leaf] (l1) at (-1.0, 4.9) {};
\node[compLeaf] (l2) at (-0.4, 4.9) {};
\node[compLeaf] (l3) at ( 0.2,4.9) {};
\node[compLeaf] (l4) at ( 1.0,4.9) {};
\node[leaf] (l5) at ( 1.6, 4.9) {};
\node[compLeaf] (l6) at ( 2.2, 4.9) {};
\node[leaf] (l7) at ( 3.0,4.9) {};
\node[leaf] (l8) at ( 3.6, 4.9) {};
\node[leaf] (l9) at ( 4.2, 4.9) {};

\node[font=\small, below=3pt of l1] {leaf};

\draw[thick] (root)--(na); \draw[thick] (root)--(nb); \draw[thick] (root)--(nc);
\draw[thick] (na)--(l1); \draw[thick] (na)--(l2); \draw[thick] (na)--(l3); \draw[thick] (nb)--(l4); \draw[thick] (nb)--(l5); \draw[thick] (nb)--(l6);
\draw[thick] (nc)--(l7); \draw[thick] (nc)--(l8); \draw[thick] (nc)--(l9);

\node[selLeaf] (sl1) at (l1) {};
\node[selLeaf] (sl2) at (l5) {};

% \draw[redarrow, bend right=30] (root.south west) to (sl1.north);
% \draw[redarrow, bend left=12]  (root.south)      to (sl2.north);
\draw[redarrow, bend right=30] (root.south west) to (na.north);
\draw[redarrow, bend left=30] (root.south) to (nb.north);
\draw[redarrow, bend left=30] (nb.south east) to (sl2.north east);
\draw[redarrow, bend right=30] (na.south west) to (sl1.north);

% S1 description
\node[font=\small, align=center, text width=55mm] at (1.85, 3.05)
  {
   At each level, fetch children of top $W$ super-blocks via PIR.\\[3pt]
   % Score defined by: $\sigma_Q(\mathcal B)=\sum\limits_{w\in Q}\sigma_w(\mathcal B)$.\\[3pt]
   Keep top-$W$ leaves (beam frontier).};

% ================================================================
%  STAGE 2 — Sub-block WAND     (panel y: 4.1..8.6)
% ================================================================
\begin{scope}[on background layer]
  \fill[panelBorder] (4.65, 4.15) rectangle (10.15, 8.6);
\end{scope}
\node[font=\small\bfseries] at (7.45, 8.3) {Stage 2: Sub-block Score Computation};

% Green containers (behind the rows)
\begin{scope}[on background layer]
  \fill[green!8, rounded corners=6pt, draw=green!40, very thick]
    (5.2, 4.35) rectangle (7.15, 8.05);
  \fill[green!8, rounded corners=6pt, draw=green!40, very thick]
    (7.75, 4.35) rectangle (9.7, 8.05);
\end{scope}

\node[font=\small\bfseries] at (6.17, 7.85) {Leaf block A};
\node[subRow] (a1) at (6.17, 7.25) {$\mathcal{SB}_1$};
\node[subRow] (a2) at (6.17, 6.35) {$\mathcal{SB}_2$};
\node[selRow] (a3) at (6.17, 5.45) {$\mathcal{SB}_3$ \checkmark};
\node[subRow] (a4) at (6.17, 4.60) {$\mathcal{SB}_4$};

\node[font=\small\bfseries] at (8.72, 7.85) {Leaf block B};
\node[subRow] (b1) at (8.72, 7.25) {$\mathcal{SB}_5$};
\node[selRow] (b2) at (8.72, 6.35) {$\mathcal{SB}_6$ \checkmark};
\node[subRow] (b3) at (8.72, 5.45) {$\mathcal{SB}_7$};
\node[subRow] (b4) at (8.72, 4.60) {$\mathcal{SB}_8$};

\draw[arrow] (sl1.east) to[out=5, in=180] (5.2, 5.45);
\draw[arrow] (sl2.east) to[out=5, in=180] (5.2, 6.35);

% S2 description
\node[font=\small, align=center, text width=52mm] at (7.45, 3.25)
  {Fetch sub-block \texttt{$W\cdot|Q|\cdot\lceil |\mathcal{B}|/s\rceil$} rows via PIR.\\[3pt]
   Compute and Rank Sub-block score locally %\\[3pt]
   to select $k_s$ promising sub-blocks.};

% ================================================================
%  STAGE 3 — Embedding Re-rank  (panel y: 4.1..8.6)
% ================================================================
\begin{scope}[on background layer]
  \fill[panelBorder] (10.3, 4.05) rectangle (15.6, 8.6);
\end{scope}
\node[font=\small\bfseries] at (12.95, 8.3) {Stage 3: Client Local Ranking};

% Doc pool background
\begin{scope}[on background layer]
  \fill[orange!6, rounded corners=6pt, draw=orange!45, very thick]
    (10.55, 5.5) rectangle (15.35, 8.0);
\end{scope}
\node[font=\small, align=center] at (12.95, 7.8)
  {Sub-block Document Embeddings};

\node[docBox]  (d1) at (11.4, 7.1) {$e_{D_1}$};
\node[selDoc]  (d2) at (12.95,7.1) {$e_{D_2}$ $\star$};
\node[docBox]  (d3) at (14.5, 7.1) {$e_{D_3}$};
\node[selDoc]  (d4) at (11.4, 6.1) {$e_{D_4}$ $\star$};
\node[docBox]  (d5) at (12.95,6.1) {$e_{D_5}$};
\node[docBox]  (d6) at (14.5, 6.1) {$e_{D_6}$};

\draw[arrow] (a3.east) to[out=0, in=180] (10.55, 6.1);
\draw[arrow] (b2.east) to[out=0, in=180] (10.55, 6.7);

% cosine arrow (from pool bottom to output box)
\draw[bluearrow] (12.95, 5.5) -- (12.95, 4.95)
  node[midway, right=4pt, font=\small]
    {client: $\cos(\mathbf{e}_Q,\mathbf{e}_D)$};

% output box
\node[draw, rounded corners=5pt, fill=teal!12, thick,
      minimum width=38mm, minimum height=9mm,
      font=\normalsize, align=center]
  (out) at (12.95, 4.55) {Output: top-$k$ doc IDs};

% S3 description
\node[font=\small, align=center, text width=48mm] at (12.95, 3.25)
  {Retrieve embeddings via PIR;\\[3pt]
   client ranks by $\cos(\mathbf{e}_Q,\mathbf{e}_D)$\\[3pt]
   and returns final top-$k$ results.};

% ================================================================
%  STAGE DIVIDERS
% ================================================================
% \draw[gray!30, very thick, dashed] (4.55, 8.7) -- (4.55, 2.75);
% \draw[gray!30, very thick, dashed] (10.25,8.7) -- (10.25,2.75);

% ================================================================
%  PRIVACY FOOTER   y = 2.1
% ================================================================
% \node[draw=gray!50, rounded corners=6pt, fill=gray!6, thick,
%       minimum width=176mm, minimum height=13mm,
%       font=\small, align=center]
%   at (7.3, 1.55)
%   {\textbf{Privacy guarantee:} All PIR queries conceal which rows are fetched.
%    \quad The server observes only access sizes determined by public parameters.};

% ================================================================
%  LEGEND   y = 0.85  (single-row, centred)
% ================================================================
% \node[draw=gray!45, rounded corners=5pt, fill=white, thick, inner sep=8pt,
%       font=\small]
%   at (7.3, 1.45) {%
%   \begin{tabular}{@{}cl@{\quad\quad}cl@{\quad\quad}cl@{}}
%     \raisebox{-3pt}{\tikz\node[selLeaf, minimum size=7mm]{};}
%       & Selected leaf (beam) &
%     \raisebox{-3pt}{\tikz\node[selRow, minimum width=16mm, minimum height=7mm]{};}
%       & Selected sub-block &
%     \raisebox{-3pt}{\tikz\node[selDoc, minimum width=16mm, minimum height=7mm]{};}
%       & Top-ranked document
%   \end{tabular}};
\node[draw=gray!45, rounded corners=5pt, fill=white, thick, inner sep=8pt,
      font=\small]
  at (7.3, 1.45) {%
  \begin{tabular}{@{}cl@{\quad\quad}cl@{\quad\quad}cl@{\quad\quad}cl@{}}
    \raisebox{-3pt}{\tikz\node[selLeaf, minimum size=7mm]{};}
      & Top node &
    \raisebox{-3pt}{\tikz\node[compLeaf, minimum size=7mm]{};}
      & Retrieved node (beam) &
    \raisebox{-3pt}{\tikz\node[selRow, minimum width=16mm, minimum height=7mm]{};}
      & Selected sub-block &
    \raisebox{-3pt}{\tikz\node[selDoc, minimum width=16mm, minimum height=7mm]{};}
      & Top-ranked document
  \end{tabular}%
};

\end{tikzpicture}
}
\caption{Three-stage \tree protocol illustration. The interaction are all private under PIR.}
\label{illu:bm25tree}
\end{figure}

% \begin{figure}
% \centering
% \includegraphics[width=0.8\textwidth]{figures/BM25Tree_interactive.pdf}
% \caption{2-stage \tree protocol illustration for a binary tree with beam search width $W=1$. Purple lines denote PIR interaction.}
% \label{illu:bm25tree}
% \end{figure}

\noindent\textbf{Solving Block-Based BM25 Challenges.}
To resolve \circled{1}, we organize the blocks recursively into a tree
whose leaves are the individual blocks and whose internal nodes each
represent a \emph{super-block} spanning the documents of all blocks in
their subtree. 
At every node we store an array of $|\vocab|$ entries,
where the $i$-th entry holds the maximum lexical score of term $i$ over
all documents in that super-block. 
For full generality, we consider $r$-ary trees, i.e., branching factor $r$, so as to reduce the height of the tree.
Next, we descend the tree with a \emph{beam search}~\cite{bisiani1987beam}
of width $W$, confine the next level traversal to the children of the highest scoring $W$ nodes. 
Since the previous level advances only $W$ super-blocks (and
each super-block has $r$ children), layer $i$ considers only $W \times r$ candidate super-blocks. For each
candidate, the algorithm sums the upper bounds of the query terms in $Q$
and retains the $W$ highest-scoring super-blocks of the level. 
It is worth noting that given that beam search is a heuristic technique, it doesn't guarantee to find the true top-$k$ blocks (but our experiments show better-than-SOTA accuracy). 
Overall, the traversal descends
$\lceil \log_r (|\corpus| / |\mathcal{B}|) \rceil$ levels and examines at
most $W\cdot r$ nodes per level, for a total of
$O\!\left(W\cdot r \cdot |Q| \cdot \log_r \frac{|\corpus|}{|\mathcal{B}|}\right)$ upper-bound array
fetches. The baseline, in contrast, scans blocks linearly and requires
$|Q| \frac{|\corpus|}{|\mathcal{B}|}$ array fetches, i.e., from linear to logarithmic. 
Regarding \circled{2}, the \client proceeds interactively: at each level \client
retrieves via PIR the relevant array entries from the candidate super-blocks that \client already identified, sums them locally to obtain
the total upper bound for each candidate super-block, and uses the result to select which
super-blocks become candidates at the next level. 
Finally, the bottom level stores not the embeddings themselves but the identifiers of the embeddings contained in each block.
Retrieving the embeddings therefore requires one additional round of PIR, after which the client performs the dense retrieval step locally. 
Overall, the $\client$ performs $\log_r \frac{|\corpus|}{|\mathcal{B}|}$ PIR rounds, and at each round it batches $W\cdot r \cdot |Q|$ array indices to retrieve (here, we concatenate all arrays-of-upper-bounds of the level into a single one). 
We apply four further optimizations to improve accuracy and performance. 
First, at preprocessing time $\server$ clusters the corpus so that each block contains mutually similar documents. 
Second, each PIR entry at a given node $\mathcal B$ and word $w$ stores the scores of $w$ w.r.t. $r$ children of $\mathcal B$ rather than to $\mathcal B$ itself. At query time, this reduces the PIR cost from $|Q|Wr$ queries over a $|\vocab|\times r^{\ell+1}$-sized child-level PIR database to $|Q|W$ queries over a $|\vocab|\times r^\ell$-sized current-level PIR database.
Third, we add a refinement layer below each leaf block, separate from the tree traversal. Each leaf contains $|\mathcal{B}|$ documents and is partitioned into sub-blocks of $s$ documents, giving branch factor $r^\star=\lceil |\mathcal{B}|/s\rceil$. After the tree traversal filters the search space to $W\times |\mathcal{B}|$ candidate documents, the refinement layer selects the top $k_s$ sub-blocks, reducing the embedding-fetch cost from $|\mathcal{B}|\times W$ embeddings to $k_s\times s$.
Finally, to avoid storing $|\mathcal{V}|$ mostly-zero upper bounds per super-block at each level, we realize each level's PIR database as a cuckoo hash table holding only non-zero entries, trading $m$ PIR queries per lookup for a substantial reduction in server storage and PIR cost (see Appendix~\ref{app:cuckoo}). 
There are many parameters that affect search speed and quality, such as branching factor, $r$, beam width, $W$, block size, $|\mathcal{B}|$, etc. we study the effect and the interplay of the different parameters in Appendix~\ref{app:tree-experiment}. 
A high-level design of \tree is shown in Figure ~\ref{illu:bm25tree}, where we split the design of \tree into three stages: 
In Stage $1$, the \client traverses the tree with beam search, using one PIR round per level, until it reaches $W$ leaf blocks.
In Stage $2$, the \client retrieves via PIR the upper bounds of the sub-blocks of these leaf blocks, selects the top $k_s$ sub-blocks, and obtains the identifiers of their embeddings.
In Stage $3$, the \client retrieves these embeddings via PIR and performs the dense retrieval locally.

\section{Evaluation}
\label{sec:evaluation}
\subsection{Experimental Setup}
\label{sec:setup}

We implement our designs $\bin$ and $\tree$ (Section~\ref{sec:PILLAR_hybrid_retrieval}) in Go and compare them against PACMANN~\citep{pacmann}, the state-of-the-art PPRAG protocol, which answers each query through a multi-round private ANN traversal. 
For all three protocols, we run a hyperparameter grid search over all parameters (we detail the tunable parameters in Table~\ref{tab:hyperparams}). %
We plot/detail the metrics for every configuration of every method over our chosen datasets in Appendix~\ref{app:experiments}. 
Because the number of configurations is large, this section reports five representative configurations per protocol; Appendix~\ref{app:experiments} repeats each plot with all configurations included.

\textbf{On Selecting  Representative Configurations.} 
We adapt a standard algorithm for determining how to minimize the metric \emph{cost} while maximizing the metric \emph{benefit} over multiple configurations, called \emph{Kneedle}~\cite{satopaa2011kneedle}.
 We define the benefit of a configuration by the mean of its MRR@10 and Recall@10, each min-max normalized first, which are standard metrics~\cite{beirThakur2021BEIRAH} in determining the quality of documents retrieved.
The ideal cost for a PPRAG must consider two quantities, ($i$) the number of PIR rounds it needs, and ($ii$) the total data transferred.
We use per-query WAN latency as a proxy for both: each sequential PIR round adds a full network round trip, and each byte transferred adds transmission time over the limited bandwidth. 
Following Kneedle, for each of the three protocols, we compute and select five configurations from the Pareto frontier: The lowest-cost configuration (C1), the configuration with the largest normalized benefit relative to its normalized cost (C2), the configurations with the highest MRR@10 (C3) and highest Recall@10 (C4), and the highest-cost configuration (C5).
Appendix~\ref{app:experiments} details the full procedure.
For completeness, we consider other instantiations of the metric cost, such as LAN latency, computation time, and total data sent, in Appendix~\ref{app:experiments}.

\textbf{Datasets.}
We evaluate on the MS MARCO passage retrieval benchmark~\citep{bajaj2018msmarco}, consisting of $8.8$ million passages and $6{,}980$ queries, and on SciFact~\cite{wadden-etal-2020-fact_scifact}, a corpus of $5{,}183$ documents and $300$ queries.
\begin{wraptable}[24]{r}{0.36\textwidth}
\centering\small
\setlength{\tabcolsep}{4pt}
\renewcommand{\arraystretch}{1.1}
\begin{tabularx}{\linewidth}{@{}r|>{\raggedright\arraybackslash}X@{}}
\toprule
\textbf{Param.} & \textbf{Meaning} \\
\midrule
\multicolumn{2}{@{}c@{}}{\textbf{\bin}} \\
\cmidrule{1-2}
$R$ & Hash table size \\
\cmidrule{1-2}
$t$ & \# Documents per bin \\
\cmidrule{1-2}
$\database_{emb}$ & Retrieve embeddings directly or use two PIR rounds \\
\midrule
\multicolumn{2}{@{}c@{}}{\textbf{\tree}} \\
\cmidrule{1-2}
$r$ & Branching factor of the tree \\
\cmidrule{1-2}
$|\mathcal{B}|$ & \# Documents per leaf block \\
\cmidrule{1-2}
$s$ & \# Documents per sub-block \\
\cmidrule{1-2}
$W$ & Beam width \\
\midrule
\multicolumn{2}{@{}c@{}}{\textbf{PACMANN}} \\
\cmidrule{1-2}
$p$ & \# Steps in ANN traversal \\
\cmidrule{1-2}
$n$ & \# Neighbors retrieved per step \\
\bottomrule
\end{tabularx}
\caption{The parameter grid searched for each protocol. Every combination was run on both
SciFact and MS MARCO.}
\label{tab:hyperparams}
\end{wraptable}
These two datasets cover the qualities that matter for a PPRAG evaluation:
MS MARCO is large, so latency, communication, and memory all matter, and its queries are open-domain, so exact term overlap is a weak measure of similarity, which works against the lexical scoring components in both \bin and \tree.
SciFact is three orders of magnitude smaller, so cost differences matter less, but its queries are scientific claims with rare, specific terminology, where lexical scoring is strongest.

\textbf{Testbed.} Retrieval experiments were run on a Debian 12 container with a 64-core, 128-thread AMD EPYC Zen 2 CPU, but we limit each process to only 15 threads. 
For end-to-end answer quality, the documents retrieved by each selected configuration are passed to an LLM, Qwen2.5-7B-Instruct~\cite{qwen2025qwen25technicalreport}. 
The resulting answers are scored by the RAGAS framework~\cite{es2025ragasautomatedevaluationretrieval}, with Llama-3-70B~\cite{grattafiori2024llama3herdmodels} as judge.
To simulate WAN latency, we use a round-trip-time of 50 milliseconds and 400 megabits per second bandwidth. 
See further details in~\ref{app:testbed}

\subsection{Comparison against PACMANN and Ablation}
\label{sec:comparison}

\begin{figure}[tbp]
    \centering
    \begin{subfigure}[b]{0.45\textwidth}
        \centering
        \includegraphics[width=\textwidth]{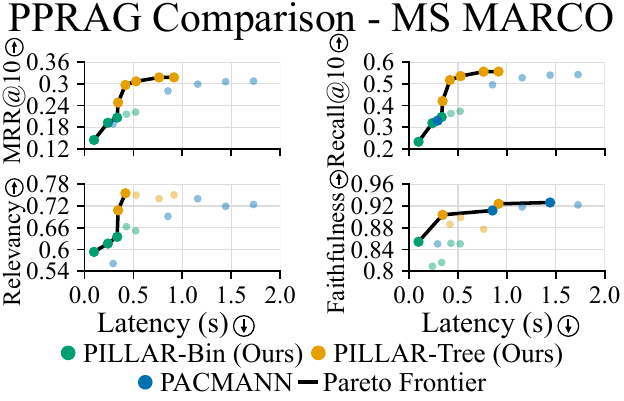}
        \caption{PPRAG Comparison - MS MARCO}
        \label{fig:msmarco}
    \end{subfigure}
    \hfill
    \begin{subfigure}[b]{0.45\textwidth}
        \centering
        \includegraphics[width=\textwidth]{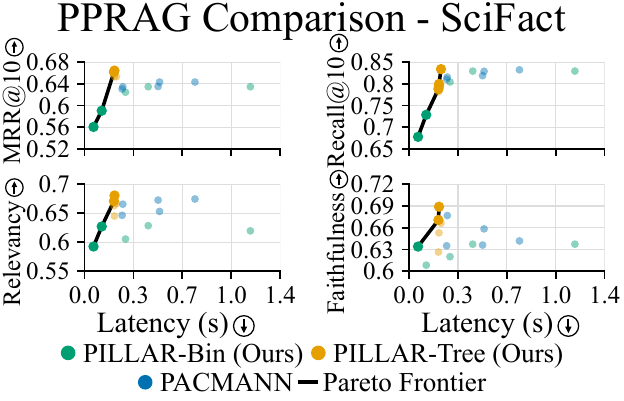}
        \caption{PPRAG Comparison - SciFact}
        \label{fig:SciFact}
    \end{subfigure}
    \caption{Retrieval quality (MRR@10, Recall@10) and answer quality (Answer Relevancy,
    Faithfulness) (Y-axis, higher is better) against per-query WAN latency (X-axis, lower is better) for the five selected configurations of each protocol. 
    The black line is the Pareto frontier, which highlights the best configurations.
    Configurations not on the Pareto frontier are slightly transparent.}
    \label{fig:pacmann_comparison}
\end{figure}

\textbf{Metrics.}
To benchmark our proposed PPRAGs, we use Recall@10, which measures whether the correct document appears among the top $10$ retrieved, and MRR@10, which rewards placing the correct document near the top.
We also use RAGAS metrics, \emph{Faithfulness} measures whether the claims in the LLM's answer are supported by the retrieved documents, and \emph{Answer Relevancy} measures whether the LLM actually answered the question, regardless of the presence of the correct document. 
We report both groups, because MRR and Recall alone do not show whether the retrieved context is useful to an LLM\footnote{The evaluation of PACMANN only considers MRR and Recall, but a RAG pipeline needs to retrieve documents that generate a well-informed answer, so we report RAGAS metrics alongside MRR and Recall}. 
As all of our configurations offer a cost-benefit trade-off, we show the \emph{Pareto frontier} (black line) in Figure~\ref{fig:pacmann_comparison} for each configuration in our two datasets. 
If a configuration lies on the Pareto frontier, no configuration of any protocol achieves better quality (Y-axis) at equal or lower latency (X-axis).
Across all of our experiments (See Figure~\ref{fig:pacmann_comparison}), $93\%$ of the points on the Pareto frontier come from our proposed \proto hybrid architecture and only $7\%$ come from PACMANN.

\textbf{On the Analysis of The Pareto Frontier.}
$\bin$ retrieves the top $10$ documents in as little as $0.098$\,s on MS MARCO and $0.065$\,s on SciFact. 
There is no configuration of PACMANN (on the Pareto frontier) that runs under $0.1$\,s on either datasets (while \bin has multiple),  the fastest PACMANN configuration takes $0.293$\,s on MS MARCO and $0.270$\,s on SciFact, which is $3\times$ to $4\times$ slower. 
As for \tree, in terms of speed, $\tree$ rarely reaches $\bin$'s latencies, since its traversal needs $\lceil\log_r(|\corpus|/|\mathcal{B}|)\rceil$ rounds.
In terms of quality, across all configurations on the Pareto frontier in Figure~\ref{fig:pacmann_comparison} that run in under $1.0$\,s, \tree achieves the highest quality across all metrics and across both datasets.
Across all configurations used to produce Figure~\ref{fig:pacmann_comparison}, $54\%$ of the points are configurations of \tree.

\textbf{$\tree$ Latency by Stages.}
Many \tree configurations retrieve documents that score highly across quality metrics, but, for some of them, the latency is increased.
To investigate this, we show how each of the three stages of $\tree$ affects latency in Figure~\ref{fig:tree_ablation}. %
On MS MARCO, Stages $1$ and $3$ stay roughly constant, which are the traversal and embedding retrieval stages, respectively.
Almost all latency variability comes from Stage $2$, the sub-block retrieval stage.
Stage $ 2$ latency drives retrieval quality; the highest-latency configuration (C5) reaches $0.318$ MRR on MS MARCO, while the lowest-latency configuration reaches $0.248$, a $28\%$ gain.
$W$ determines how many candidate documents Stage $2$ refines, which makes it the first parameter to tune on a new dataset.
On SciFact, the five configurations differ by under $20$\,ms. This is due to the size of the tree, which will always be  smaller on a small dataset (such as SciFact).

\label{sec:ablation}
\begin{figure}[t]
    \centering
    \begin{subfigure}[t]{0.48\textwidth}
        \centering
        \includegraphics[width=\textwidth]{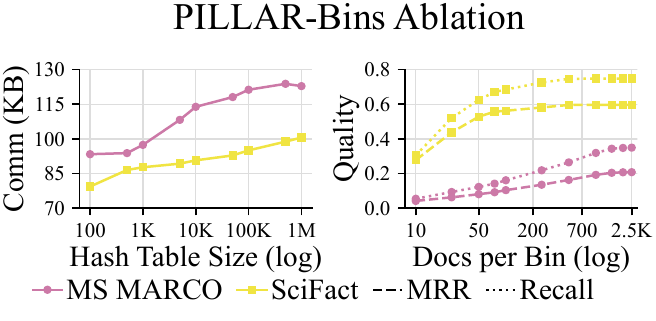}
        \caption{$\bin$ Ablation}
        \label{fig:bins_ablation}
    \end{subfigure}
    \hfill
    \begin{subfigure}[t]{0.48\textwidth}
        \centering \includegraphics[width=0.965\textwidth]{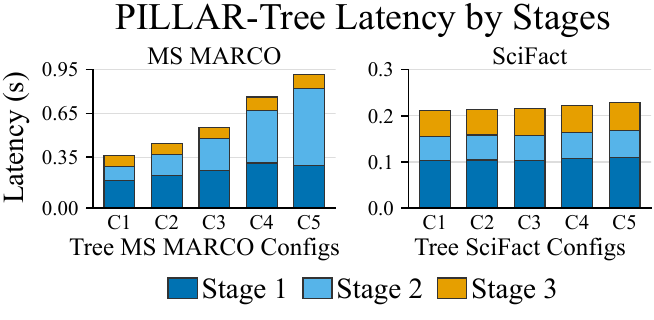}
        \caption{$\tree$ Latency by Stages}
        \label{fig:tree_ablation}
    \end{subfigure}
    \caption{Left: Two $\bin$ ablation plots, total communication cost per query against the size of the hash table (log scale), and retrieval quality against documents stored in a bin (log scale).
    Right: per-query latency of the five selected $\tree$ configurations (C1-C5) (defined in Section~\ref{sec:setup} and Table~\ref{tab:tree-param-tuning} lists their specific parameters) by the 3 stages defined in Section~\ref{sec:bm25tree}.}
    \label{fig:ablation}
\end{figure}

\textbf{$\bin$ Ablation.}
We test the two parameters of \bin: the size of the hash table, and the number of documents stored in a bin. 
Specifically, we investigate ($i$) how the size of the hash table affects the communication cost, as well as ($ii$) how the number of documents stored in a bin affects retrieval quality (we provide a more extensive ablation on \bin  in Appendix~\ref{app:experiments}).
The left plot of Figure~\ref{fig:bins_ablation} fixes the documents per bin at $10$ and varies the hash table size (X-axis), showing that total communication cost per query (Y-axis, Comm KB) increases with the size of the hash table (as is expected when using PIR).
We vary the size of the hash table by $6$ orders of magnitude, but see that the communication cost for the largest hash table is not even $2\times$ the cost of the smallest.
Thus, hash table size has little impact on the total communication cost, allowing flexibility to choose the hash table size that provides the best retrieval quality (as shown in Appendix~\ref{app:experiments}).
The right plot of Figure~\ref{fig:bins_ablation} fixes the hash table sizes at $1$ million for MS MARCO and $5$ thousand for SciFact and varies the number of documents stored in a bin (X-axis). 
We see that MRR@10/Recall@10 (Y-axis, Quality) for the retrieved documents increases as the number of documents stored in a bin increases.
Scifact sees diminishing returns, after storing $700$ documents per bin, there is little more retrieval quality to gain.
MS Marco continues to see benefits from more returned documents, even at the highest tested values, $1.0$K-$2.5$K documents per bin, MS MARCO still increases.

\section{Conclusion}

We presented \proto, a pair of PIR-friendly hybrid retrieval protocols ($\bin$ and $\tree$) that split retrieval between a cheap lexical score and a semantic re-ranking performed entirely by the client. 
We evaluate our two protocols using the state-of-the-art, PACMANN. \bin always has a faster configuration available, answering queries three times faster than any configuration PACMANN admits, while $\tree$ frequently retrieves better quality documents than PACMANN at the same speed.

\section*{Acknowledgements}
Daniel Blackley and Evgenios M. Kornaropoulos were partially supported by NSF awards \#2154732 and \#2439951. Truong Son Nguyen and Ni Trieu were partially supported by NSF award  \#2451972, ARPA-H award  \#1AY2AX000167-01, and Amazon Research award.

\bibliographystyle{plainnat}
\bibliography{bibs/abbrev3,bibs/crypto,bibs/ref}

\appendix

\section*{Appendix}

\section{AI use statement}
In this work, we used generative AI tools to assist in implementing methods, namely writing code for our protocols and result plots.
We have not used generative AI tools to generate synthetic datasets, develop theoretical models or conceptual frameworks, formulate mathematical claims, provide critical ingredients for proving mathematical claims, assist in the writing of proofs, propose or refine hypotheses, design or provide feedback on research methodology or experiments, clean or reformat datasets, or interpret results, and assistance with translation and qualitative or thematic data analysis are not applicable to this work. Additionally, we used generative AI tools to draft parts of the paper (initial text, section structure and titles, and \LaTeX{} layout), to edit the paper for grammar, spelling, and readability, and to identify and summarize relevant literature. We have reviewed all AI-assisted work. All AI-generated code was reviewed, re-written, tested, and validated for correctness by the authors. AI-drafted text served only as a starting point and was revised by the authors, and all technical content, claims, and proofs are the authors' own. Literature identified or summarized with AI assistance was checked against the original sources, and all citations were verified manually. We take responsibility for the final content of this work, including text, claims or artifacts produced with the aid of generative AI.

\section{Reproducibility}

We provide an extensive Appendix that details the step-by-step operation of \tree and \bin (Appendix~\ref{app:protocol_details}) and we also detail all the different parameters tested for all three protocols  (Appendix~\ref{app:experiments} and Table~\ref{tab:hyperparams}). 
All code for \bin is available at \url{https://github.com/dkblackley/bins-go} and \tree at \url{https://github.com/sonnguyenasu/bm25-tree}.

\section{Notation}
In the paper we use different hyper-parameters and variables for our protocols. We detailed the parameters notations and their corresponding definition in Table ~\ref{tab:notation}. %

\begin{table}[t]
    \centering
    \begin{tabular}{r|l}
    \toprule
        Parameter & Definition \\
    \midrule
        $\corpus$ & The corpus \\
        $\mathcal V$ & Set of all unique terms in $\corpus$ \\
        $d$ & Embedding vector dimension \\
        $D$ & Documents in $\corpus$ \\
        $e_D$ & Embedding vector of document $D$ \\
        $Q$ & Client query \\
        $e_Q$ & Embedding vector of client query \\
        $k$ & Number of relevant document retrieved \\
        $\ind(\cdot)$ & Function to represent the index of a document/ document block \\
    \midrule
        $t$ & Number of documents returned per query term \\
        $R$ & Number of rows used in \bin \\
        $H$ & The public hash function used in \bin \\
    \midrule        
        $B$ & Block Size \\
        $\mathcal B$ & Block of $B$ documents in the corpus \\
        $\sigma_w(\mathcal B)$ & Largest score between term $w$ and any document in $\mathcal B$ \\
        $r$ & Branching factor of the tree used in \tree \\
        $\mathcal H$ & Tree height \\
        $W$ & Beam width used in \tree's beam search \\
        $s$ & Sub-block size \\
        $q^\star$ & Fixed query size, attained by pad/truncate the client query \\
        $k_s$ & Number of sub-block with document embeddings being retrieved\\
    \bottomrule
    \end{tabular}
    \caption{Definition of variables used in the paper}
    \label{tab:notation}
\end{table}

\section{Additional Experiment Results}
\label{app:experiments}

In this section we present additional result of all configurations that we use some of them in the main paper. 

\subsection{Hyperparameter Tuning on \tree Configurations}
\label{app:tree-experiment}
We perform a grid search on $\tree$ hyperparameters and evaluate recall, MRR@10, and estimated number of PIR calls needed. From the evaluation results, we select one representative operating points from the Pareto frontier shown in Figure ~\ref{fig:bm25tree-pareto} using a defined process detailed later, together with the four configurations attaining the lowest PIR call, the highest PIR calls, the highest MRR@10 and highest Recall@10. We then run a full evaluation on these 5 configurations: we (1) evaluate the MRR@10 and Recall@10 of the configurations on a test set of $1,500$ queries from MSMARCO train set for MSMARCO, and a test set of $75$ queries from Scifact train set for Scifact-- to make sure tuning set and test set are different -- (2) execute each configuration under PIR to benchmark runtime and (3) use RAGAS to evaluate downstream generation quality. For MSMarco, the search is performed over $B\in \{16,32,64\}, r\in\{8,32,128\}, q^\star \in \{2,4,6,8\}, W\in\{50, 100,200,400\}, s\in \{4, 8, 16\}$. For Scifact, the search is performed over $B\in \{8,16,32,64\}, r\in\{4,8,32,128\}, q^\star \in \{4,8,12,16\}, W\in\{5, 10,20,40,80\}, s\in \{2,4,8\}$. The set of 5 candidate configurations along with their corresponding MRR are shown in Table ~\ref{tab:tree-param-tuning}.

\subsubsection{Representative Configurations Selection.} 
As discussed earlier, we use Pareto frontier to choose a representative configuration along with two configurations represent the best recall@10 and best MRR@10. Here we show how we choose the five configurations.

\paragraph{Pareto Frontier} Let $\Theta$ denote the evaluated hyperparameter configurations. Following
the standard definition of Pareto optimality
\cite{miettinen1999nonlinear}, a configuration $\theta'$ dominates
$\theta$, denoted $\theta' \succ \theta$, if
\[
\mathsf{MRR}(\theta') \geq \mathsf{MRR}(\theta),\qquad
\mathsf{Recall}(\theta') \geq \mathsf{Recall}(\theta),\qquad
\mathsf{PIRCall}(\theta') \leq \mathsf{PIRCall}(\theta),
\]
with at least one strict inequality. The Pareto frontier is the set of configurations where improving one metric worsen at least another:
\[
\mathcal{P}
=
\left\{
\theta\in\Theta:
\nexists\theta'\in\Theta
\text{ such that }\theta'\succ\theta
\right\}.
\]

\paragraph{Effectiveness Score.} To visualize the effectiveness-cost trade-off, we min-max normalize each retrieval metric:
\[
\widetilde{\mathsf{MRR}}(\theta)
=
\frac{
  \mathsf{MRR}(\theta)-\mathsf{MRR}_{\min}
}{
  \mathsf{MRR}_{\max}-\mathsf{MRR}_{\min}
}, \qquad
\widetilde{\mathsf{Recall}}(\theta)
=
\frac{
  \mathsf{Recall}(\theta)-\mathsf{Recall}_{\min}
}{
  \mathsf{Recall}_{\max}-\mathsf{Recall}_{\min}
}.
\]
We define the aggregate effectiveness score as
\[
E(\theta)
=
\frac{
  \widetilde{\mathsf{MRR}}(\theta)
  +
  \widetilde{\mathsf{Recall}}(\theta)
}{2}.
\]

The effectiveness-cost envelope $\mathcal P_E$ is the set when maximizing
$E$ and minimizing $\mathsf{PIRCall}$:
\[
\begin{aligned}
\mathcal{P}_{E} = \Bigl\{ \theta\in\Theta :{}& \nexists\,\theta'\in\Theta \text{ such that } E(\theta')\geq E(\theta) \land \mathsf{PIRCall}(\theta')\leq\mathsf{PIRCall}(\theta), \\
& \text{with at least one strict inequality} \Bigr\}.
\end{aligned}
\]

Following the normalized difference construction of Kneedle
\cite{satopaa2011kneedle}, we define
\[
x(\theta)
=
\frac{
  \log\mathsf{PIRCall}(\theta)
  -
  \log\mathsf{PIRCall}_{\min}
}{
  \log\mathsf{PIRCall}_{\max}
  -
  \log\mathsf{PIRCall}_{\min}
},
\qquad
y(\theta)
=
\frac{
  E(\theta)-E_{\min}
}{
  E_{\max}-E_{\min}
}.
\]
The knee point, which represents the point on the envelope which witness the strongest gain in effectiveness with respect to its PIR cost, is defined as
\(\theta_{\mathrm{knee}}
=
\arg\max_{\theta\in\mathcal{P}_{E}}
\left(y(\theta)-x(\theta)\right).
\)

We then chooses the points C1, C2, C3, C4, C5 as:
\begin{itemize}
    \item C1: Point C1 with lowest number of PIR calls
    \item C2: Point C2 is the knee point $\theta_{\mathrm{knee}}$ %
    \item C3: Point C3 with highest MRR@10 on tune set
    \item C4: Point C4 with highest Recall@10 on tune set
    \item C5: Point C5 with highest number of PIR calls
\end{itemize}

\begin{table*}[t]
  \centering
  \small
  \setlength{\tabcolsep}{2pt}
  \begin{tabular}{llrrrrrrccccr}
    \toprule
    Dataset & Config
      & $B$ & $r$ & $q^\star$ & $W$ & $s$ & $k_s$
      & \multicolumn{2}{c}{Tuning}
      & \multicolumn{2}{c}{Test}
      & \multirow{2}{*}{\shortstack{PIR calls\\/ query}} \\
    \cmidrule(lr){9-10}\cmidrule(lr){11-12}
      & & & & & & &
      & MRR@10 & Recall@10 & MRR@10 & Recall@10 & \\
    \midrule
    \multirow{5}{*}{MS MARCO}
      & C1 & 64 & 128 & 2 & 50  & 16 & 32
      & 0.1833 & 0.3055 & 0.1962 & 0.3426 & \textbf{472} \\
      & C2 & 64 & 128 & 4 & 50  & 16 & 64
      & 0.2833 & 0.4880 & 0.3190 & 0.5751 & 944 \\
      & C3 & 16 & 128 & 8 & 400 & 8  & 256
      & \textbf{0.3218} & 0.5639 & 0.3551 & 0.6573 & 13616 \\
      & C4 & 64 & 128 & 8 & 400 & 4  & 256
      & 0.3211 & \textbf{0.5649} & 0.3549 & \textbf{0.6631} & 13216 \\
      & C5 & 16 & 8   & 8 & 400 & 16 & 256
      & 0.3213 & 0.5625 & \textbf{0.3556} & 0.6580 & 28352 \\
    \midrule
    \multirow{5}{*}{SciFact}
      & C1 & 64 & 128 & 4  & 5  & 8 & 5
      & 0.5303 & 0.5909 & 0.5852 & 0.6333 & \textbf{53} \\
      & C2 & 16 & 128 & 12 & 10 & 4 & 10
      & 0.6623 & 0.8008 & 0.6192 & 0.7600 & 346 \\
      & C3 & 16 & 128 & 16 & 20 & 4 & 10
      & \textbf{0.6697} & 0.8043 & \textbf{0.6350} & 0.7867 & 778 \\
      & C4 & 32 & 128 & 12 & 80 & 2 & 40
      & 0.6606 & \textbf{0.8428} & 0.6225 & 0.8000 & 2032 \\
      & C5 & 8  & 4   & 16 & 80 & 4 & 80
      & 0.6518 & 0.8267 & 0.6249 & \textbf{0.8067} & 6992 \\
    \bottomrule
  \end{tabular}
  \caption{Dense-only configurations selected from the plaintext parameter
  search. C1 has the lowest padded PIR cost, C2 is the
  effectiveness--cost knee, C3 has the highest tuning MRR@10, C4 has the
  highest tuning Recall@10, and C5 has the highest padded PIR cost. Tuning
  results use 6,980 MS MARCO queries and 300 SciFact queries; test results use
  1,500 MS MARCO queries and 75 held-out SciFact queries. Bold metric values
  are column maxima within each dataset, while bold PIR costs are column
  minima. PIR-query counts assume constant-work padding and include two
  Cuckoo calls for every Stage~1 and Stage~2 lookup and one query for each
  retrieved embedding row. They count primitive PIR queries rather than
  network round trips after batching.}
  \label{tab:tree-param-tuning}
\end{table*}

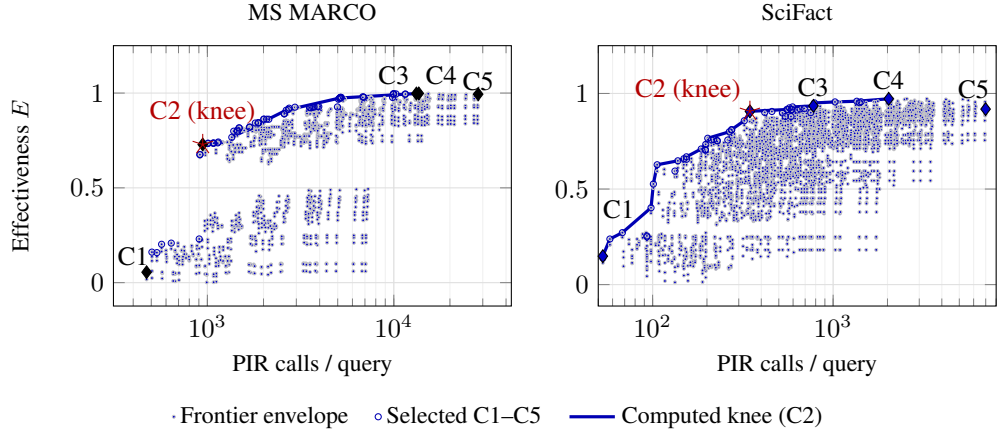
\begin{figure*}[t]
  \centering
  \begin{tikzpicture}
    \begin{groupplot}[
      group style={group size=2 by 1,horizontal sep=1.15cm},
      width=0.49\textwidth,
      height=0.36\textwidth,
      xmode=log,
      log basis x=10,
      grid=both,
      major grid style={gray!25},
      minor grid style={gray!12},
      tick label style={font=\normalsize},
      label style={font=\small},
      title style={font=\small},
      legend style={
        font=\small,
        draw=none,
        fill=none,
        /tikz/every even column/.append style={column sep=0.8em},
      },
    ]

\nextgroupplot[
  title={MS MARCO},
  xlabel={PIR calls / query},
  ylabel={Effectiveness $E$},
  ymax=1.25,
  legend to name=bm25tree-pareto-legend,legend columns=-1,
]
\addplot+[only marks,mark=*,mark size=0.55pt,
  draw=gray!45,fill=gray!35]
  table[x=cost,y=effectiveness] {data/msmarco_all.dat};
\addplot+[only marks,mark=o,mark size=1.15pt,
  draw=blue!70!black]
  table[x=cost,y=effectiveness] {data/msmarco_pareto.dat};
\addplot+[blue!70!black,very thick,mark=none]
  table[x=cost,y=effectiveness] {data/msmarco_envelope.dat};
\addlegendentry{Frontier envelope}
\addplot+[only marks,mark=diamond*,mark size=2.4pt,black]
  table[x=cost,y=effectiveness] {data/msmarco_selected.dat};
\addlegendentry{Selected C1--C5}
\addplot+[only marks,mark=star,mark size=4pt,
  draw=red!70!black,fill=red!65]
  table[x=cost,y=effectiveness] {data/msmarco_knee.dat};
\addlegendentry{Computed knee (C2)}
\node[font=\normalsize,above] at (axis cs:400,0.04) {C1};
\node[font=\normalsize,above,text=red!70!black] at (axis cs:944,0.7886466756) {C2 (knee)};
\node[font=\normalsize,above left] at (axis cs:13616,0.9982213353) {C3};
\node[font=\normalsize,above right] at (axis cs:13216,0.9975979642) {C4};
\node[font=\normalsize,above] at (axis cs:28352,0.9676750583) {C5};

\nextgroupplot[
  title={SciFact},
  xlabel={PIR calls / query},
  xmin=50,
  xmax=8000,
  ymax=1.25,
  legend columns=-1,
]
\addplot+[only marks,mark=*,mark size=0.55pt,
  draw=gray!45,fill=gray!35,forget plot]
  table[x=cost,y=effectiveness] {data/scifact_all.dat};

\addplot+[only marks,mark=o,mark size=1.15pt,
  draw=blue!70!black,forget plot]
  table[x=cost,y=effectiveness] {data/scifact_pareto.dat};

\addplot+[blue!70!black,very thick,mark=none,forget plot]
  table[x=cost,y=effectiveness] {data/scifact_envelope.dat};

\addplot+[only marks,mark=diamond*,mark size=2.4pt,black,forget plot]
  table[x=cost,y=effectiveness] {data/scifact_selected.dat};

\addplot+[only marks,mark=star,mark size=4pt,
  draw=red!70!black,fill=red!65,forget plot]
  table[x=cost,y=effectiveness] {data/scifact_knee.dat};

\node[font=\normalsize,above] at (axis cs:65,0.3) {C1};
\node[font=\normalsize,above left,text=red!70!black] at (axis cs:346,0.9064738439) {C2 (knee)};
\node[font=\normalsize,above] at (axis cs:778,0.9350880480) {C3};
\node[font=\normalsize,above] at (axis cs:2032,0.9720858375 ) {C4};
\node[font=\normalsize,above] at (axis cs:6000,0.9246055427) {C5};

    \end{groupplot}
  \end{tikzpicture}
  \par\smallskip
  \ref{bm25tree-pareto-legend}
  \caption{Effectiveness-PIR-cost trade-off over the complete plaintext grid.
  Blue marks are all tested configurations and circles are configurations
  nondominated in MRR@10, Recall@10, and padded PIR calls. The solid curve is
  the upper envelope $\mathcal P_E$, C1,C2,C3,C4,C5 are chosen configurations.}
  \label{fig:bm25tree-pareto}
\end{figure*}

\begin{table*}[t]
  \centering
  \small
  \setlength{\tabcolsep}{3pt}
  \begin{tabular}{llrrrccr}
    \toprule
    Dataset & Config
    & \texttt{bs} & \texttt{dpb} &
    MRR@10 & Recall@10
      & \multirow{2}{*}{\shortstack{WAN time \\/ query (s)}} \\
      & & & 
      & & & \\
    \midrule
    \multirow{5}{*}{MS MARCO}
      & C1 &  $10^6$ & 250
      & 0.145 & 0.2337 & \textbf{0.098} \\
      & C2 &  $10^6$ & 1500
      & 0.2067 & 0.3474 & 0.334 \\
      & C3 &  $10^6$ & 2000
      & 0.2160 & 0.3635 & 0.429 \\
      & C4 &  $10^6$ & 2500
      & \textbf{0.2219} & \textbf{0.3739} & 0.523 \\
      & C5 &  $10^6$ & 1000
      & 0.1925 & 0.3192 & 0.239 \\
    \midrule
    \multirow{5}{*}{SciFact}
      & C1 &  $10^6$ & 50
      & 0.5607 & 0.6781 & \textbf{0.065} \\
      & C2 &  $10^6$ & 250
      & 0.5905 & 0.7287 & 0.124 \\
      & C3 &  100 & 1500
      & 0.6245 & 0.8042 & 0.294 \\
      & C4 & 100 & 2500
      & \textbf{0.6346} & \textbf{0.8292} & 1.187 \\
      & C5 &  100 & 2500
      & \textbf{0.6346} & \textbf{0.8292} & 0.456 \\
    \bottomrule
  \end{tabular}
  \caption{Bins configurations selected from the plaintext parameter search.}
  \label{tab:bins-param-tuning}
\end{table*}

\begin{table*}[t]
  \centering
  \small
  \setlength{\tabcolsep}{3pt}
  \begin{tabular}{llrrccr}
    \toprule
    Dataset & Config
    & \texttt{bs} & \texttt{dpb} &
    MRR@10 & Recall@10
      & \multirow{2}{*}{\shortstack{WAN time \\/ query (s)}} \\
      & & & 
      & & & \\
    \midrule
    \multirow{5}{*}{MS MARCO}
      & C1 & 5  & 48
      & 0.1892 & 0.3306 & \textbf{0.293} \\
      & C2 & 15 & 40
      & 0.2802 & 0.4955 & 0.853 \\
      & C3 & 20 & 48
      & 0.2990 & 0.5275 & 1.157 \\
      & C4 & 30 & 48
      & \textbf{0.3073} & \textbf{0.5422} & 1.727 \\
      & C5 & 25 & 48
      & 0.3057 & 0.5397 & 1.442 \\
    \midrule
    \multirow{5}{*}{SciFact}
      & C1 & 5  & 40
      & 0.6304 & 0.8104 & \textbf{0.270} \\
      & C2 & 5  & 48
      & 0.6345 & 0.8154 & 0.275 \\
      & C3 & 10 & 40
      & \textbf{0.6432} & 0.8287 & 0.537 \\
      & C4 & 15 & 32
      & \textbf{0.6432} & \textbf{0.8321} & 0.791 \\
      & C5 & 10 & 32
      & 0.6350 & 0.8187 & 0.527 \\
    \bottomrule
  \end{tabular}
  \caption{PACMANN configurations selected from the plaintext parameter
  search.}
  \label{tab:pacmann-param-tuning}
\end{table*}

PACMANN and \bin follow a similar protocol to \tree, but instead they \emph{directly} use the recorded WAN time as the cost. We detail PACMANN's best configs in Table~\ref{tab:pacmann-param-tuning} and \bin's best in Table~\ref{tab:bins-param-tuning}.

\subsection{Extensive comparisons}
\label{app:comparison-experiments}
In this section, we plot every different configuration we ran across PACMANN, \bin and \tree. We use the same Pareot frontiers as were in Section~\ref{sec:evaluation}. We have over $2{,}000$ total configurations, which is why we limited our comparison previously.
To keep each plot legible, we only show the entire Pareto frontier and circled points on the frontier with a black outline.
To get Faithfulness and Answer Relevancy for every single different configuration using the GPUs available to us would've taken a substantial amount of time, as a result, we decided to include more total configurations but with more basic metrics. Every plot here only uses MRR and Recall

\begin{figure}[h]
    \centering
    \begin{subfigure}{0.49\linewidth}
        \includegraphics[width=\linewidth]{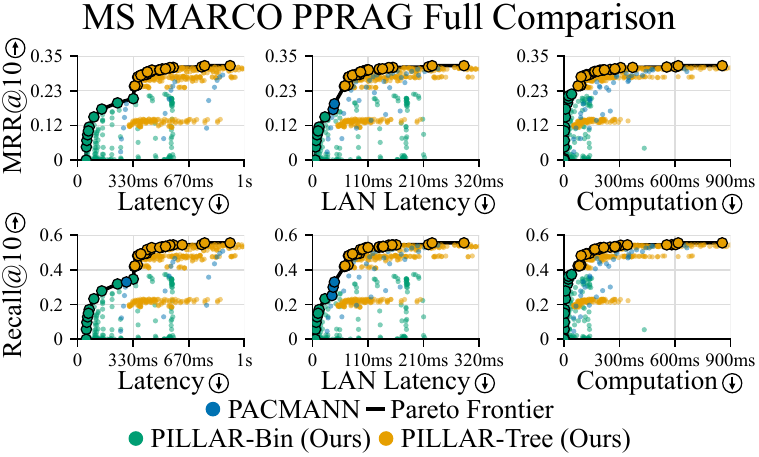}
        \caption{MS MARCO}
        \label{fig:pareto_time_msmarco}
    \end{subfigure}\hfill
    \begin{subfigure}{0.49\linewidth}
        \includegraphics[width=\linewidth]{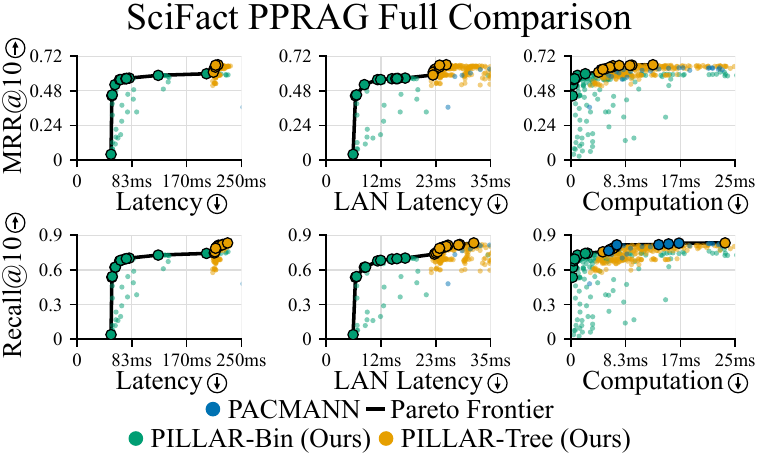}
        \caption{SciFact}
        \label{fig:pareto_time_scifact}
    \end{subfigure}
    \caption{MRR@10 and Recall@10 against per-query WAN latency, LAN latency, and computation time. Highlighted points lie on the global Pareto frontier across all three methods.}
    \label{fig:pareto_time}
\end{figure}

\begin{figure}[t]
    \centering
    \begin{subfigure}{0.49\linewidth}
        \includegraphics[width=\linewidth]{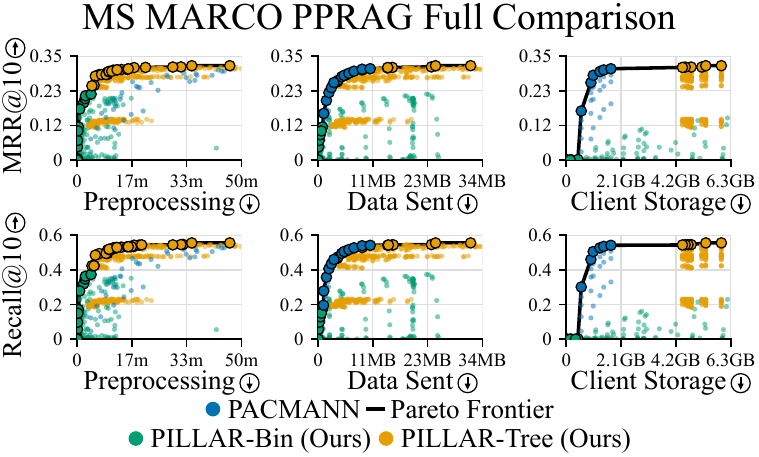}
        \caption{MS MARCO}
        \label{fig:pareto_storage_msmarco}
    \end{subfigure}\hfill
    \begin{subfigure}{0.49\linewidth}
        \includegraphics[width=\linewidth]{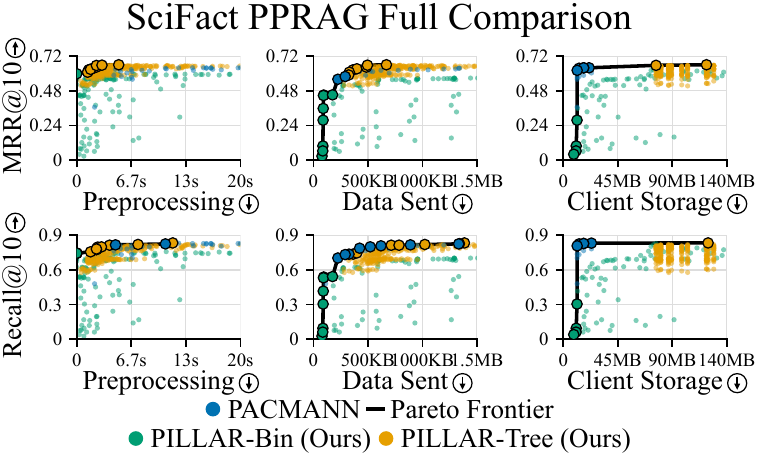}
        \caption{SciFact}
        \label{fig:pareto_storage_scifact}
    \end{subfigure}
    \caption{MRR@10 and Recall@10 against one-time PIR preprocessing time, per-query data sent, and one-time client storage. Highlighted points lie on the global Pareto frontier across all three methods.}
    \label{fig:pareto_storage}
\end{figure}

\textbf{On the Time Costs.}
Figure~\ref{fig:pareto_time} reports quality against per-query WAN latency, LAN latency, and computation time.
If we fix the budget at the $x$-axis limits ($1$\,s WAN, $320$\,ms LAN, and $900$\,ms computation on MS MARCO; $250$\,ms, $35$\,ms, and $25$\,ms on SciFact), \tree achieves the highest MRR@10 and Recall@10 in every panel.
\bin occupies the low-latency end of every frontier. On SciFact, \bin already exceeds $0.5$ MRR@10 within $83$\,ms of WAN latency and accounts for $50.0\%$ of all frontier points.
PACMANN has no configuration on the Pareto frontier in eight of the twelve panels, so it is never the Pareto-optimal choice. These eight panels are MRR@10 under WAN latency and both metrics under computation on MS MARCO, and every panel on SciFact except Recall@10 under computation. In the remaining four panels, PACMANN contributes at most $5$ points.
Across all $377$ frontier points, $97.1\%$ are configurations of \bin or \tree.

\textbf{On the Preprocessing, Communication, and Storage Costs.}
Figure~\ref{fig:pareto_storage} reports quality against PIR preprocessing time, data sent, and client storage. All costs are per query except PIR preprocessing and client storage, which are for the entire end-to-end run.
\emph{Preprocessing.} PACMANN has no configuration on the Pareto frontier on MS MARCO or for MRR@10 on SciFact. This is because its additional rounds of PIR queries require correspondingly more PIR preprocessing.
At the $x$-axis limits ($50$\,min on MS MARCO, $20$\,s on SciFact), \tree achieves the highest quality on both metrics.
\emph{Data sent.} At the $x$-axis limits ($34$\,MB on MS MARCO, $1.5$\,MB on SciFact), \tree again achieves the highest quality. 
PACMANN occupies the middle of the frontier on MS MARCO ($51.9\%$ of points for MRR@10, $51.7\%$ for Recall@10) but only $11.8\%$ and $39.1\%$ on SciFact.
Even though PACMANN may often be the choice for a low amount of total communication, as we have shown previously, the many rounds practically drag down WAN and LAN latency.
\emph{Client storage.} PACMANN's many rounds of interaction let it keep the client-side cost low, and it reaches near-maximal quality with under $2.1$\,GB on MS MARCO and under $45$\,MB on SciFact, whereas \tree's highest-quality configurations require over $4.2$\,GB and $90$\,MB, respectively.
The cost of these rounds is paid in latency, as shown in Figure~\ref{fig:pareto_time}.
Despite this, PACMANN does not hold a majority of the frontier, it only has $35.3\%$ and $33.3\%$ of the points on MS MARCO (tied with \tree), and $45.5\%$ and $50.0\%$ on SciFact. \tree still achieves the highest quality at the $x$-axis limits ($6.3$\,GB and $140$\,MB).
Moreover, on MS MARCO nearly every configuration that is viable of every method requires client storage on the order of gigabytes, so PACMANN's advantage does not change the scale of the client's cost.
Across all $262$ frontier points in Figure~\ref{fig:pareto_storage}, $75.6\%$ are configurations of \bin or \tree, with \tree the largest share on both datasets ($48.0\%$ and $41.6\%$).

\subsection{Further \bin ablation plots}
\label{app:bins_ablation}
\begin{figure}[h]
    \centering
    \begin{minipage}{0.49\linewidth}\centering
        \includegraphics[width=\linewidth]{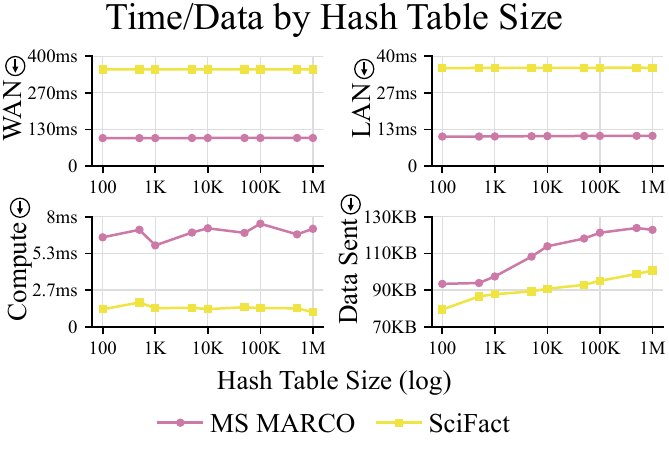}
    \end{minipage}\hfill
    \begin{minipage}{0.49\linewidth}\centering
        \includegraphics[width=\linewidth]{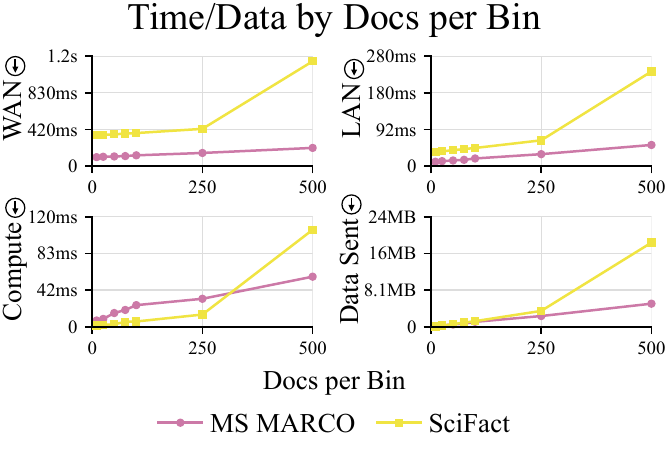}
    \end{minipage}
    \caption{$\bin$ per-query cost. Left: varying the hash table size. Right: varying the number of documents per bin.}
    \label{fig:app_bins_cost}
\end{figure}

\begin{figure}[t]
    \centering
    \begin{minipage}{0.49\linewidth}\centering
        \includegraphics[width=\linewidth]{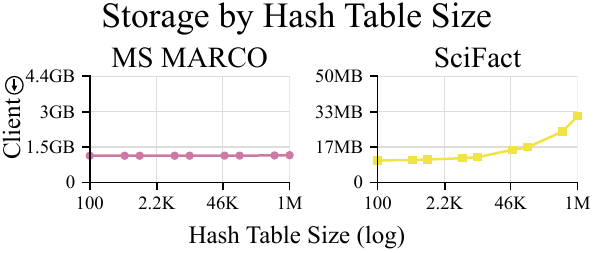}
    \end{minipage}\hfill
    \begin{minipage}{0.49\linewidth}\centering
        \includegraphics[width=\linewidth]{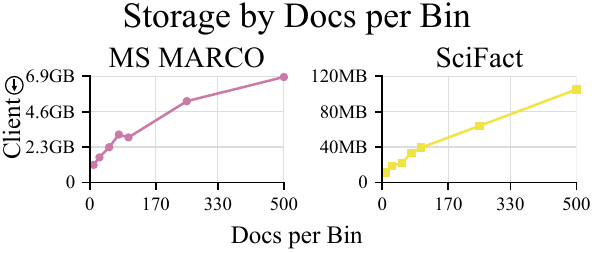}
    \end{minipage}
    \caption{$\bin$ worst-case client storage. Left: varying the hash table size. Right: varying the number of documents per bin.}
    \label{fig:app_bins_storage}
\end{figure}

\begin{figure}[t]
    \centering
    \begin{minipage}{0.49\linewidth}\centering
        \includegraphics[width=\linewidth]{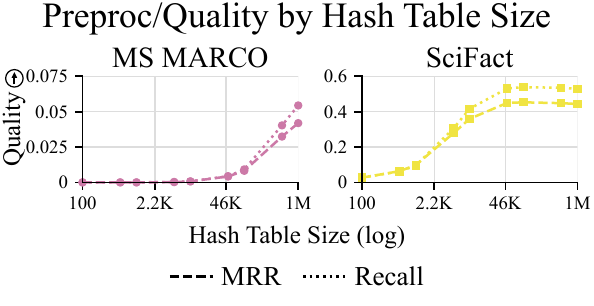}
    \end{minipage}\hfill
    \begin{minipage}{0.49\linewidth}\centering
        \includegraphics[width=\linewidth]{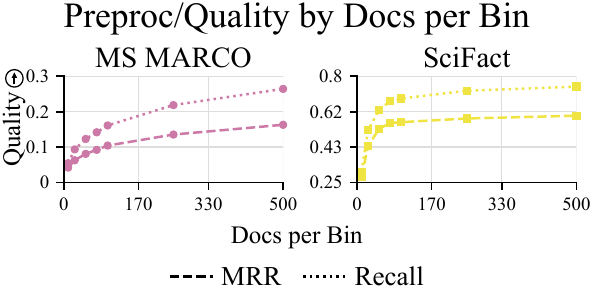}
    \end{minipage}
    \caption{$\bin$ PIR preprocessing time (top) and retrieval quality (bottom). Left: varying the hash table size. Right: varying the number of documents per bin.}
    \label{fig:app_bins_preproc}
\end{figure}

We extend the $\bin$ ablation of Section~\ref{sec:evaluation} with system costs, client storage and preprocessing time.
As in the main text, when varying the hash table size we fix the number of documents per bin at $10$, and when varying the number of documents per bin we fix the hash table size at $1$M for MS MARCO and $5$K for SciFact.
WAN latency, LAN latency, compute (total running time) and data sent (upload and download) are per query. Client storage and PIR preprocessing total costs over the entire run.

\textbf{WAN, LAN and total compute.}
Figure~\ref{fig:app_bins_cost} (left) shows that WAN latency, LAN latency and compute are either not affected by the hash table size increases or show no consistent trend.
Only data sent increases, the same trade-off we discussed in Section~\ref{sec:evaluation}
Figure~\ref{fig:app_bins_cost} (right) shows that all four costs increase with the number of documents per bin.
For SciFact, the increase is small up to $250$ documents per bin and steep between $250$ and $500$.

\begin{wrapfigure}[22]{l}{0.49\linewidth}
    \centering
    \includegraphics[width=\linewidth]{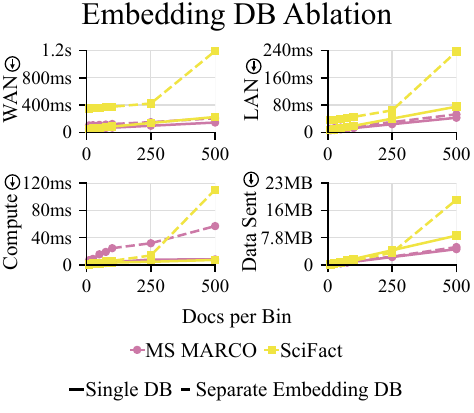}
    \caption{$\bin$ per-query cost when embeddings are stored in the bins (Single DB) or in a separate database (Separate Embedding DB), varying the number of documents per bin.}
    \label{fig:app_bins_embdb}
\end{wrapfigure}
\textbf{Storage costs and PIR preprocessing.}
Figure~\ref{fig:app_bins_storage} reports client storage in the worst case, where every hint is completely full. In practice, client storage is dependent on how many bins are full and ranges from $30$-$50$\% better.
For MS MARCO, client storage is approximately $1$GB and does not change with the hash table size.
As we saw in the ablation comparing against PACMANN, \bin that take up a large client hint tend to not perform well regardless. 
For SciFact, client storage increases from $11$MB at a hash table size of $100$ to $30$MB at $1$M, with most of the increase above $10$K.
Unlike the hash table size, the number of documents per bin has a large effect on client storage, which increases for SciFact and for MS MARCO.Figure~\ref{fig:app_bins_preproc} (top) shows that, surprisingly, PIR preprocessing time does not depend on the hash table size.
In our implementation, we made heavy use of multithreading to perform pre-processing, so it is likely that the bottleneck was thrashing related, as a large hash table size (should) imply more entries to retrieve PIR hints from.
Preprocessing time increases with the number of documents per bin, from $0.6$ minutes at $10$ to $4.9$ minutes at $500$ for MS MARCO, and from a few seconds at up to $100$ to $1.7$ minutes at $500$ for SciFact.
This implies that the most important factor is not actually the size of the bin, but instead the size of the entry.
Figure~\ref{fig:app_bins_preproc} (bottom left) shows that retrieval quality increases with the hash table size.
For MS MARCO, MRR@10 and Recall@10 remain below $0.01$ up to $100$K and reach $0.042$ and $0.054$ at $1$M, still increasing at the largest tested size.
For SciFact, MRR@10 and Recall@10 increase up to $50$K, reaching $0.45$ and $0.53$, and see no notable benefit.

\textbf{1 or 2-round PIR}
$\bin$ can store the document embeddings directly in the bins and return them with a single PIR query (Single DB), or store document IDs in the bins and retrieve the corresponding embeddings from a separate database with a second PIR query (Separate Embedding DB).
Figure~\ref{fig:app_bins_embdb} shows that Single DB has lower or comparable WAN latency, LAN latency, compute and data sent at every tested number of documents per bin, and that the gap grows with the number of documents per bin.
The difference is largest for SciFact: at $500$ documents per bin, Single DB reduces WAN latency from $1.2$s to $240$ms, LAN latency from $235$ms to $77$ms, compute from $110$ms to $7$ms, and data sent from $19$MB to $7.8$MB.
Single DB is usually cheaper because it removes the second round of PIR, the most expensive component of the protocol.
Future protocols should always aim to minimize PIR rounds.

\section{Protocol Details}
\label{app:protocol_details}

This section gives the detailed protocol figures for \bin and \tree, together with some implementation details that Section~\ref{sec:PILLAR_hybrid_retrieval} suppresses for clarity.
The first concerns what a RAG pipeline actually returns. A RAG pipeline consumes document \emph{text}, whereas the dense stage of \bin outputs embeddings, and an embedding does identify the document it came from.
PACMANN and \tree are unaffected, since both retrieve document indices in their traversals, but \bin does not.
Each entry of the database held by \server in the \bin protocol stores an index alongside the embedding, using an extra $32$ bits per row, and the protocol outputs indices that \client uses in the final PIR round of protocol~\ref{fig:protocol-pprag}. Figure~\ref{fig:protocol-bin-1round} gives the construction in full.

There is another concern regarding storage used by the client. 
As described in Section ~\ref{sec:PILLAR_hybrid_retrieval}, the database in \bin doesn't store one document per row, like PACMANN or \tree, but instead $t$ document embeddings per row. Under a PIR scheme with preprocessing, \client retains a collection of entries in the database, that scale sub-linearly with the database size, so this falls on client storage rather than on the server alone (see Appendix~\ref{app:tree-experiment}).
We therefore provide a second variant that splits retrieval into two rounds, the first recovers identifiers from a database of $32$-bit sized rows, and the second recovers the corresponding embeddings from a database holding each document embedding once. Figure~\ref{fig:protocol-bin-2round} gives this variant,
which trades one additional round for a client state independent of $R$ and $t$.
 
 Figure ~\ref{fig:bm25-tree-protocol} presents the formal construction of \tree protocol.
Finally, {Figure ~\ref{fig:protocol-pprag} shows the whole PPRAG pipeline using the \bin/\tree constructions.}

\begin{figure}[ht]
\fbox{%
  \begin{minipage}{0.96\linewidth}
    \vspace{0.4em}
    \textbf{Protocol 1.a: \bin (single-round variant)}
    \vspace{0.6em}

    \noindent\textbf{Participants.} Server \server and client \client.

    \vspace{0.5em}
    \noindent\textbf{Public parameters.} Corpus $\corpus=\{D_1,\dots,D_N\}$; analyzer $\analyze$; embedding function $\emb$ of dimension $\embdim$; BM25 parameters $(k_1,b)$; row count $R$; row depth $t$; hash function $H:\vocab\rightarrow\{0,\dots,R-1\}$; batch bound $\beta$; output size $k$; reserved dummy index $\bot$.

    \vspace{0.5em}
    \noindent\textbf{Client input.} Query $Q$ with $\analyze(Q)=(w_1,\dots,w_{|Q|})$.

    \vspace{0.6em}
    \hrule
    \vspace{0.6em}

    \noindent\textbf{Preprocessing (server \server).}
    \begin{enumerate}
        \item $\server$ initialize empty database $\database$ of $R$ rows.
        \item For each $w\in\vocab$, \server scores every $D\in\corpus$ containing $w$ under BM25
        with $w$ as a single-term query, and sets $\mathcal{L}_w$ to the $t$ highest-scoring
        documents in descending order, padded with $\langle\bot,\mathbf{0}\rangle$ if fewer than
        $t$ documents contain $w$. Write $\mathcal{L}_w[j]$ for its $j$th document.
      \item \server groups terms by row: $w$ belongs to row $z = H(w)$.
        Write $W_z=\{w\in\vocab : H(w)= z\}$ and let $m_z = |W_z|$ be the number of terms that collide in row $z$, and fix an arbitrary order
        $w^{(1)},\dots,w^{(m_z)}$ on them, so that
        $\mathcal{C}_z = \{\mathcal{L}_{w^{(1)}}, \dots, \mathcal{L}_{w^{(m_z)}}\}$ are the ordered lists competing for row $z$.
      \item \server fills row $z$ by interleaving the lists of $\mathcal{C}_z$ round-robin, one
        rank at a time: the $j$th document of the $i$th list goes to position $(j-1)m_z + i$:
         $\database[z][(j-1)m_z + i] \leftarrow \mathcal{L}_{w^{(i)}}[j]$. The row therefore
        holds the best document of every colliding term, then the second-best of every term,
        and so on.
      \item \server discards duplicates, which arise when a document appears in two colliding
        lists, and then truncates each row to its first $t$ documents. A row with $m_z=0$ is left empty.
      \item \server then replaces all document text with embeddings and an index: For every row index $0 < z \le R$ and offset $0 < x \le t$, \server retrieves $D = \database[z][x]$, computes $e_D = \emb(D)$ and sets $\database[z][x] = \langle \ind(D), e_D \rangle$.
      \item \server runs the PIR preprocessing phase over $\database_{\emb}$ with \client, who stores the resulting private state.
    \end{enumerate}

    \vspace{0.6em}
    \noindent\textbf{Retrieval.}
    \begin{enumerate}
      \item \client computes $\analyze(Q)$ and the row indices $r_i\leftarrow H(w_i)$ for $i\le|Q|$, removing repeated indices.
      \item \client fixes the batch to exactly $\beta$ indices and if $|Q|>\beta$ it retains the first $\beta$ or if $|Q|<\beta$ they pad with random indices from $\{0,\dots,R-1\}$.
      \item \client sends a single batched PIR query for these $\beta$ indices and recovers the corresponding $\beta$ rows of $\database_{}$ from \server's response and its private state.
      \item \client locally discards entries with index $\bot$ and duplicate indices, forming a candidate pool $\mathcal{D}_{\mathrm{cand}}$ of at most $\beta t$ embeddings.
      \item \client computes $\mathbf{e}_Q=\emb(Q)$ locally and ranks $\mathcal{D}_{\mathrm{cand}}$ by $\cos(\mathbf{e}_Q,\mathbf{e}_{D})$.
      \item \client outputs the $k$ document indices of highest similarity to the query.
    \end{enumerate}

    \vspace{0.4em}
  \end{minipage}%
}
\caption{Formal specification of \bin (single-round variant). \server observes exactly $\beta$ batched PIR queries, and holds a database $\database$ of $R$ rows. 
Each row is an ordered list of $t$ entries and an entry is a pair $\langle \ind(D), \emb(D)\rangle$ of a $32$-bit index and a $\embdim$-dimensional vector.}
\label{fig:protocol-bin-1round}
\end{figure}

\begin{figure}[ht]
\fbox{%
  \begin{minipage}{0.96\linewidth}
    \vspace{0.4em}
    \textbf{Protocol 1.b: \bin (two-round variant)}
    \vspace{0.6em}

    \noindent\textbf{Participants.} Server \server and client \client.

    \vspace{0.5em}
    \noindent\textbf{Public parameters.} Identical to Protocol 1.a, with a second batch bound $\beta'$.

    \vspace{0.5em}
    \noindent\textbf{Client input.} Query $Q$ with $\analyze(Q)=(w_1,\dots,w_{|Q|})$.

    \vspace{0.6em}
    \hrule
    \vspace{0.6em}

    \noindent\textbf{Preprocessing (server \server).}
    \begin{enumerate}
      \item \server builds $\database_{}$ exactly as in Protocol 1.a preprocessing, steps 1--4.
      \item \server replaces all document text with a unique index only. For every row index $0 < z \le R$ and offset $0 < x \le t$, \server retrieves $D = \database[z][x]$ and sets $\database[z][x] = \ind(D)$, the index alone.
      \item For each $D_j\in\corpus$, \server initializes a database $\database_{\emb}$ by writing $e_D = \emb(D_j)$ to row $j$ of $\database_{\emb}$.
      \item \server runs the PIR preprocessing phase over $\database_{}$ and over $\database_{\emb}$ with \client, who stores the resulting private states.
    \end{enumerate}

    \vspace{0.6em}
    \noindent\textbf{Retrieval.}
    \begin{enumerate}
       \item \client computes retrieval steps 1--2 of Protocol 1.a exactly.
      \item \client sends one batched PIR query to $\database_{}$, and recovers $\beta$ rows of indices.
      \item \client locally discards $\bot$ and duplicate indices, obtaining a set $I$ of, at most, $\beta \cdot t$ distinct indices.
      \item \client fixes $I$ to exactly $\beta'$ indices: if $|I|>\beta'$ it retains the first $\beta'$ indices, if $|I|<\beta'$ it appends indices drawn randomly from $\{1,\dots,R\}$.
      \item \client sends a single batched PIR query using these $\beta'$ indices to $\database_{\emb}$ and recovers the corresponding embeddings.
      \item \client computes $\mathbf{e}_Q=\emb(Q)$ locally and ranks the documents of $I$ by $\cos(\mathbf{e}_Q,\mathbf{e}_{D})$, ignoring the embeddings returned for padding indices.
      \item \client outputs the $k$ document indices of highest similarity to the query.
    \end{enumerate}

    \vspace{0.4em}
  \end{minipage}%
}
\caption{Formal specification of \bin{} (two-round variant). The \server stores two databases, one database $\database_{}$ of $R$ rows of $t$ indices each, and a database $\database_{\emb}$ of $N$ rows, row $j$ holding the single embedding $\emb(D_j)$. The primary difference is that the server only has to hold one embedding per entry, instead of $t$ embeddings per entry.}
\label{fig:protocol-bin-2round}
\end{figure}

\begin{figure}[t]
\small
\fbox{
\begin{minipage}{0.96\linewidth}
\textbf{Protocol 1: BM25-Tree — Privacy-Preserving Lexical Pruning with Dense Ranking}

\vspace{0.5em}
\textbf{Participants.}  
Server $\server$ and client $\client$

\vspace{0.5em}
\textbf{Public Parameters.}  
Corpus $\corpus$ of $N$ documents:$ \corpus=\{D_1,\dots,D_N\}$; $\vocab$ be set of all unique terms within all the documents in $\corpus$; leaf block size $B$, sub-block size $s$, tree branching factor $r$, beam width $W$, $\mathcal T$ be set of all tree nodes, $\mathcal T_\ell$ be set of tree nodes on level $\ell$, stage 2 number of retrieved sub-block $k_s$, maximum unique query terms $q^\star$, $m$ public hash functions $h_1,\dots h_m: \vocab\times\mathcal T \to \{0,1\}^\star$; Cuckoo Hash table load factor $\alpha$ ; target result size $k$;  %PIR scheme $\Pi.

\vspace{0.5em}
\textbf{Client Input.}  
Query $Q = \{w_1,\ldots,w_{q^\star}\}$.

\vspace{0.5em}
\textbf{Server state (initialized).}  
\begin{enumerate}
    \item Let $\mathcal H=1+\lceil \log_r(N/B)\rceil$ be height of $r$-ary tree. For each level $\ell\in[\mathcal H]$, $\server$ computes $N_\ell$ the number of pairs $(w,\mathcal B)\in \vocab\times\mathcal T_\ell$ such that $\sigma_w(\mathcal B)\neq0$. \server then initializes a database $\database_{tree, \ell}$ of $\lceil N_\ell/\alpha\rceil$ rows for each $\ell\in[\mathcal H]$.
    \item \server initialize an empty database $\database_{vector}$ of $\lceil N/s\rceil$ rows.
\end{enumerate}

% \vspace{0.5em}
% \textbf{Client Output.}  
% Top-$k$ document identifiers ranked by cosine similarity.

\vspace{0.75em}
\hrule
\vspace{0.75em}

\textbf{Preprocessing (server \server)}
\begin{enumerate}
    \item For each node $\mathcal B_{i,\ell}$ in level $\ell$ for $1\leq \ell \leq \mathcal H-1$
    \begin{enumerate}
        \item \server computes $\sigma_w(\mathcal B_{i,\ell+1,j})$ for child node $\mathcal B_{i,\ell+1,j}$ of $\mathcal B_{i,\ell}$ ($j\in[r]$) and all words $w$ that appears in documents in $\mathcal B_{i,\ell}$. 
        \item For each word $w\in\mathcal B_{i, \ell}$, the tuple $(w, \ind(\mathcal B_{i,\ell}), \sigma_w(\mathcal B_{i,\ell+1,1}),\dots,\sigma_w(\mathcal B_{i,\ell+1,r}))$  of $r+2$ fields are added to the database $\database_{tree,\ell}$ using Cuckoo Hash rule on input $(w,\ind(\mathcal B_{i,\mathcal H}))$ over $m$ hash functions $h_1,\dots,h_m$.
    \end{enumerate}
    \item For each leaf block $\mathcal{B}$, \server computes the score of their sub-blocks $\mathcal {SB}_{1},\dots,\mathcal {SB}_{\lceil B/s\rceil}$ of size $s$ with respect to each word $w$ appearing in the block, store $(w,\ind(\mathcal B),\sigma_w(\mathcal{SB}_1),\dots,\sigma_w(\mathcal{SB}_{\lceil B/s\rceil}))$ to $\database_{tree,\mathcal H}$ using Cuckoo Hash on input $(w,\ind(\mathcal B))$ and $m$ hash functions $h_1,\dots,h_m$.
    \item For each sub-block $\mathcal{SB}$ at position $p=\ind(\mathcal {SB})$, store at row $p$ of $\database_{vector}$ the tuple $(e_{D_1},\dots,e_{D_s})$ the embedding vectors of $s$ documents $D_1,\dots,D_s$ within $\mathcal{SB}$. 
\end{enumerate}

\vspace{0.75em}
\hrule
\vspace{0.75em}
\end{minipage}
}
\caption{The BM25-Tree protocol for privacy-preserving lexical pruning with dense re-ranking (Part I: setup and preprocessing).}
\label{fig:bm25-tree-protocol}
\end{figure}
% \vspace{0.5em}

% Part II: online retrieval. \ContinuedFloat preserves the figure number.
\begin{figure}[t]
\ContinuedFloat
\small
\fbox{
\begin{minipage}{0.96\linewidth}
\textbf{Protocol 1: BM25-Tree (continued)}

\vspace{0.75em}
\hrule
\vspace{0.75em}

\textbf{Stage 1: Tree Traversal}
\begin{enumerate}
  \item \client initializes $F_0=\{\mathcal B_{1,0}\}$, where
  $\mathcal B_{1,0}$ is the root node of the tree.
  \item For each level $\ell=0,\dots,\mathcal H-1$ do
  \begin{enumerate}
      \item Let
      $F_\ell=
      \{\mathcal B_{1,\ell}^\star,\dots,
      \mathcal B_{|F_\ell|,\ell}^\star\}$
      be the set of candidate nodes at level $\ell$. For each $w\in Q$,
      \client sends PIR to \server, retrieving rows
      \(
      \left(
      h_{a}\!\left(w,\ind(\mathcal B_{i,\ell}^\star)\right)
      \right)_
      {\substack{a\in[m]\\i\in[|F_\ell|]}}
      \)
      from $\database_{tree,\ell}$, for a total of
      $m\times|F_\ell|$ rows.
      
      \item For each $i\in[|F_\ell|]$, \client checks whether any
      retrieved row has its first two fields equal to
      $w$ and $\ind(\mathcal B_{i,\ell}^\star)$. If none does, \client sets
      \(
      \sigma_w(\mathcal B_{i,\ell+1,j})=0
      \), for all $j\in[r]$,
      where $\mathcal B_{i,\ell+1,j}$ is the $j$-th child of
      $\mathcal B_{i,\ell}^\star$. Otherwise, \client sets
      $\sigma_w(\mathcal B_{i,\ell+1,j})$ to the $(j+2)$-th field of the
      matching row.
      
      \item \client computes
      \(
      \sigma_Q(\mathcal B_{i,\ell+1,j})
      =
      \sum_{w\in Q}
      \sigma_w(\mathcal B_{i,\ell+1,j})
      \)
      for every child of every node in $F_\ell$.
      
      \item \client locally ranks these child nodes according to
      $\sigma_Q$ and lets $F_{\ell+1}$ be the set of the top
      $\min\{W,|\mathcal T_{\ell+1}|\}$ nodes.
  \end{enumerate}
  \item \client outputs $F_{\mathcal H}=
  \{\mathcal B_{1,\mathcal H}^\star,\dots,\mathcal B_{W,\mathcal H}^\star\}$ as top $W$ leaf block.
\end{enumerate}

\textbf{Stage 2: Sub-Block Retrieval}
\begin{enumerate}
  \item For each $w\in Q$ do
  \begin{enumerate}
      \item \client sends PIR to \server, retrieving rows $(h_{i\in [m]}(w,\ind(\mathcal B_{1,\mathcal H}^\star)),\dots, h_{i\in[m]}(w,\ind(\mathcal B_{W,\mathcal H}^\star)))$ ($m\times W$ rows total) from $\database_{tree, \mathcal H}$
      \item For each $i\in[W]$, \client checks if any of the retrieved data has the first two field equal to $w$ and $\ind(\mathcal{B}_{i,\mathcal H}^\star)$. If there are non, \client sets $\sigma_w(\mathcal{SB}_{i,\mathcal H,j})=0$ for all sub-block $\mathcal{SB}_{i,\mathcal H,j}$ of $\mathcal{B}_{i,\mathcal H}$, else \client matches corresponding $\sigma_w(\mathcal{SB}_{i,\mathcal H,j})$ with the $j+2$-th retrieved value of the corresponding row matches $(w,\mathcal{B}_{i,\mathcal H})$.
  \end{enumerate}
  \item \client computes $\sigma_Q(\mathcal{SB}_{i,\mathcal H,j})=\sum\limits_{w\in Q}\sigma_w(\mathcal{SB}_{i,\mathcal H,j})$ for all sub-blocks of top-$W$ leaf block.
  \item \client locally rank and get set of top $k_s$ sub-blocks $\mathcal{SB}_1^\star,\dots,\mathcal{SB}_{k_s}^\star$
\end{enumerate}

\vspace{0.5em}

\textbf{Stage 3: Dense Retrieval and Ranking}

\begin{enumerate}
  \item \client sends PIR query to \server, retrieving rows $(\ind(\mathcal{SB}_1^\star),\dots,\ind(\mathcal{SB}_{k_s}^\star))$ from $\database_{vector}$
  \item Let $\mathcal D_{cand} = \{D_1^\star,\dots,D_{k_s\times s}^\star\}$ be set of $k_s\times s$ documents whose embedding vectors are retrieved. \client locally computes $\cos(e_Q, e_{D_j^\star})$ for all $j\in[k_s\times s]$ and ranks the score
  \item Client output top $k$ documents in $\mathcal D_{cand}$ with highest similarity.
\end{enumerate}

\vspace{0.75em}
\hrule
\vspace{0.75em}

\end{minipage}
}
\caption{The BM25-Tree protocol for privacy-preserving lexical pruning with dense re-ranking.}
\label{fig:bm25-tree-protocol}
\end{figure}

\begin{figure}[h]
\fbox{
  \begin{minipage}{0.96\linewidth}
    \vspace{0.4em}
    \textbf{Protocol 3: Full PPRAG based on \macroname}
    \vspace{0.6em}

    \noindent\textbf{Participants.} Server \server and client \client.

    \vspace{0.5em}
    \noindent\textbf{Public parameters.} Corpus $\corpus$ of $N$ documents:$ \corpus=\{D_1,\dots,D_N\}$, Set $\mathsf{params}$ of parameters for \macroname (\bin/\tree) retrieval algorithm.

    \vspace{0.5em}
    \noindent\textbf{Client input.} Keyword query $Q=\{w_1,\dots,w_{|Q|}\}$.

    \vspace{0.5em}
    \noindent\textbf{Server state (initialized / preprocessing).}
    \begin{enumerate}
        \item \server runs \macroname initialization on $\mathsf{params}$
        \item \server runs \macroname preprocessing on $\corpus,\mathsf{params}$
        \item \server and \client preprocess $\corpus$ for PIR, each document $D_i\in \corpus$ associate with an identifier $i\in [N]$. 
    \end{enumerate}

    \vspace{0.6em}
    \hrule
    \vspace{0.6em}
    \noindent\textbf{Retrieval}
    \begin{enumerate}
      \item \client and \server participate in private retrieval \macroname protocol on $Q$ and $\corpus$, \client receive indices top-$k$ documents in $\corpus$ with respect to $Q$. Call the set $\mathsf{ID}_{top-k} = \{D_1^\star,\dots,D_k^\star\}$
      \item \client runs PIR to retrieve content of $k$ documents $D_1^\star,\dots,D_k^\star$.
    \end{enumerate}
    \vspace{0.6em}

    \noindent\textbf{Generation}
    
         \client runs local LLM to generate answer with input query $Q$ and $k$ documents $D_1^\star,\dots,D_k^\star$, output answer $\mathcal A$
    \vspace{0.4em}
  \end{minipage}
}
\caption{\macroname-based PPRAG protocol. \client and \server participates in online secure retrieval before \client runs inference on its local LLM.}
\label{fig:protocol-pprag}
\end{figure}

\section{Complexity Analysis}
\label{app:complexity}

\subsection{Notation}
We denote the cost of a batch PIR call on $\psi$ rows over a database of $\mathcal R$ rows, each row consist of $C$ fields as $\pcost(\psi,\mathcal R,C)$. For Batch PianoPIR implementation from PACMANN ~\cite{pacmann}, the cost is $\widetilde{O}(C\sqrt{\mathcal R\psi})$

\subsection{\bin Complexity}
\paragraph{Round complexity.} There are two variants of \bin. For the first variant in Figure ~\ref{fig:protocol-bin-2round}, the client runs two rounds of communication: The first round to retrieve $|Q|$ rows from the database, the second round to retrieve the document embedding corresponding to the document indices retrieved in the first round. For the second variant in Figure ~\ref{fig:protocol-bin-1round}, \client run a single round and retrieve the embeddings directly from \server and rank those $|Q|\times t$ retrieved embeddings locally.

\paragraph{PIR cost.} The number of rows being fetched are $\beta$ for the first round over a total of $R$ rows, where each row has an average of $O(|\vocab| t/R)$ document IDs, and $O(\beta' \times t)$ rows for the second round over a database of $N$ rows, each with $d$ fields. Thus, the total cost of the \bin (2-round) is: $\boxed{\pcost(\beta,R,|\vocab|t/R) + \pcost(t\beta', N, d)}$. 

For the \bin (1-round), the complexity is simply $\beta$ queries over a database of $R$ rows, each with $t$ document embeddings, resulting in total PIR cost of $\boxed{\pcost(\beta, R, td)}$

\paragraph{Local computation.} The client locally computes cosine distance between the query embedding and all retrieved document embeddings, incurring $O(|Q|td)$ computation complexity at client.

\subsection{\tree Complexity}
\paragraph{Round Complexity.} The client runs a total of $\mathcal H+1$ rounds: First $\mathcal H-1$ rounds to perform interactive tree traversal, 1 round for sub-block retrieval, and 1 round for embedding retrieval. 

\paragraph{PIR cost.} (1) For tree traversal, the number of rows being fetched on level $\ell$ is $mWq^\star$ over a database of $\lceil N_\ell/\alpha\rceil$ rows, each row consist of $r+2$ fields, resulting a cost of $\pcost(mWq^\star, \lceil N_\ell/\alpha\rceil, r+2)$. (2) For sub-block retrieval, the client fetches $mWq^\star$ over a vector of $\lceil N_{\mathcal H}/\alpha\rceil$ rows each with $\lceil B/s \rceil+2$ fields, resulting in a $\pcost(mWq^\star, \lceil N_{\mathcal H}/\alpha\rceil, \lceil B/s \rceil+2)$. (3) For embedding retrieval, the client fetches $k_s$ rows over $\database_{vector}$ of $\lceil N/s \rceil$ rows, each row has $sd$ fields, resulting $\pcost(k_s, \lceil N/s \rceil, sd)$ cost.

Hence, the total PIR cost of the \tree protocol is 
\[
\boxed{
\begin{aligned}
\pcost_{tree}=&\sum\limits_{\ell=1}^{\mathcal H-1} \pcost(mWq^\star, \lceil N_\ell/\alpha\rceil, r+2) \\&+ \pcost(mWq^\star, \lceil N_{\mathcal H}/\alpha\rceil, \lceil B/s \rceil+2) \\&+ \pcost(k_s, \lceil N/s \rceil, sd)
\end{aligned}
}
\]

\paragraph{Local computation.} The client local computation involves computing block score at each level between the input query and retrieved blocks' children. This costs $O(mWr q^\star)$ at each level from $\ell=2,\dots, \mathcal H-1$ (the first level has only 1 node so it does not require any computation). At sub block level, client performs the subblock score computation, costing $O(mW\lceil B/s\rceil q^\star)$. Finally, the client computes the cosine similarity between his query and $k_s\times s$ retrieved document embeddings, resulting in $O(k_ssd)$ computation.

Thus, the client local computation costs $\boxed{O(mWr\mathcal H q^\star+ mW\lceil B/s\rceil q^\star + k_ssd)}$

\section{Detailed Security Analysis}
\label{app:security}

We prove that both protocols provide security against a semi-honest server.
Security is formally defined via the standard simulation paradigm.

We first define computationally private PIR ideal functionality as:

\begin{definition}[Computationally Private PIR ~\citep{1997_kushilevitz_replication_not_needed_single_pir,computationalPIR}]\label{def:pir} %
A single-server PIR scheme $\Pi_{PIR} = (\mathsf{Query}, \mathsf{Answer},
\mathsf{Decode})$ over a database $DB$ of $N$ records satisfies:
\begin{enumerate}
\item \textbf{Correctness:} For any index $i\in [N]$:
\[\Pr[\mathsf{Decode}(\mathsf{Answer}(DB,\mathsf{Query}(i)))=DB[i]]
\ge 1-\mathsf{negl}(\lambda).\]
\item \textbf{Computationally Private:} for all indices $i, j \in [N]$ and all PPT
distinguishers $\mathcal{D}$:
\[
  \Bigl|\Pr\bigl[\mathcal{D}(\mathsf{Query}_\Pi(i)) = 1\bigr]
  - \Pr\bigl[\mathcal{D}(\mathsf{Query}_\Pi(j)) = 1\bigr]\Bigr|
  \leq \mathsf{negl}(\lambda).
\]
\end{enumerate}
\end{definition}

Before stating the main theorems, we establish a composition lemma that
underpins the security of \tree's multi-round interaction.

\begin{lemma}[Sequential PIR Composition]
\label{lem:composition}
Let $\Pi_{PIR}$ be a computationally private PIR scheme.
Let $T = T(\lambda)$ be a polynomial.
Any protocol that issues a sequence of $T$ independent PIR queries, where
$T$ is determined solely by public parameters, remains computationally private
against a semi-honest server.
\end{lemma}

\begin{proof}
Define a sequence of hybrid experiments $\mathsf{H}_0, \mathsf{H}_1, \ldots,
\mathsf{H}_T$ where $\mathsf{H}_\ell$ denotes the distribution in which the
first $\ell$ PIR queries are replaced by simulated queries
$\mathsf{Query}_\Pi(u_i)$ for independently uniform $u_i \xleftarrow{\$} [N]$,
while the remaining $T - \ell$ queries are real.
$\mathsf{H}_0$ is the real server view; $\mathsf{H}_T$ is the fully simulated
view.

By Definition~\ref{def:pir}, for each $\ell \in [T]$ and every PPT
distinguisher $\mathcal{D}$:
\[
  \bigl|\Pr[\mathcal{D}(\mathsf{H}_{\ell-1}) = 1]
  - \Pr[\mathcal{D}(\mathsf{H}_\ell) = 1]\bigr| \leq \mathsf{negl}(\lambda).
\]
Summing over $T = \mathrm{poly}(\lambda)$ hybrids, the total distinguishing
advantage is at most $T \cdot \mathsf{negl}(\lambda) = \mathsf{negl}(\lambda)$,
completing the proof.
\end{proof}

\setcounter{theorem}{0}
\begin{theorem}[Security of \bin (1-round)]
\label{thm:bin1}
Suppose $\Pi_{PIR}$ is a computationally private PIR scheme
(Definition~\ref{def:pir}).
Then Protocol~1 (\bin, Figure~\ref{fig:protocol-bin-1round}
) is secure against a semi-honest server.
\end{theorem}

\begin{proof}
\textbf{Server's real view.}
The server $\server$ performs all preprocessing locally 
and therefore already knows the full hash table.
During retrieval, $\server$'s view consists solely of the $\beta$ PIR queries
$\{\mathsf{Query}_\Pi(r_i)\}_{i=1}^{\beta}$ issued by the client, where
$r_i = H(w_i) \bmod R$ for keyword $w_i\in\vocab\cup\{\bot\}$ in the padded/truncated $Q$.
No other message passes from $\client$ to $\server$.

\textbf{Simulation.}
Construct $\mathsf{Sim}$ as follows: given only the public parameters
$(N, R, k, \lambda)$ and the fixed query length $\beta$, output
$\{\mathsf{Query}_\Pi(u_i)\}_{i=1}^{\beta}$ where each
$u_i \xleftarrow{\$} [R]$ is independently uniform.

\textbf{Indistinguishability.}
The simulated view consists of $\beta$ uniformly random PIR queries; the real view
consists of $\beta$ PIR queries for indices $r_1, \ldots, r_{\beta} \in [R]$
determined by the query keywords.
By Lemma~\ref{lem:composition} (applied with database size $R$ and $T = \beta$
queries), the real and simulated views are computationally indistinguishable,
completing the proof.
\end{proof}

\begin{theorem}[Security of \bin (2-round)]
\label{thm:bin2}
Suppose $\Pi_{PIR}^{(1)},\Pi_{PIR}^{(2)}$ is a computationally private PIR scheme
(Definition~\ref{def:pir}) on the index database $\database$ and the vector database $\database_{\emb}$ respectively.
Then Protocol~1 (\bin, Figure~\ref{fig:protocol-bin-2round}
) is secure against a semi-honest server.
\end{theorem}

\begin{proof}
\textbf{Server's real view.}
The server $\server$ performs all preprocessing locally 
and therefore already knows the full hash table.
During retrieval,
\begin{itemize}
\item \textbf{First round:} $\server$'s view consists solely of the $\beta$ PIR queries
$\{\mathsf{Query}^{(1)}_\Pi(r_i)\}_{i=1}^{\beta}$ issued by the client, where
$r_i = H(w_i) \bmod R$ for keyword $w_i\in\vocab\cup\{\bot\}$ in the padded/truncated $Q$.
\item \textbf{Second round:} $\server$'s view consists of $\beta'$ queries $\{\mathsf{Query}^{(2)}_\Pi(\zeta_i)\}_{i=1}^{\beta'}$ where $\zeta_i\in I^\star$ \footnote{$I^\star$ denotes the padded/truncated version of $I$ with $\beta'$ items, as described in Retrieval Step 4, Figure ~\ref{fig:protocol-bin-2round}} to the embedding database $\database_{\emb}$ issued by the client, where $\beta'$ is a known public parameter.
\end{itemize}

No other message passes from $\client$ to $\server$. We note that de-duplication and removal of $\bot$ on retrieval step 3 of Figure ~\ref{fig:protocol-bin-2round} are happening locally at client for internal score ranking, and $\client$ still sends a constant payload of $\beta'$ queries in step 4-5 regardless of deduplication/removal results of step 3.

\textbf{Simulation.}
Construct $\mathsf{Sim}$ as follows: given only the public parameters
$(N, R, k, \lambda)$ and the fixed query lengths $\beta,\beta'$, output
$\{\mathsf{Query}^{(1)}_\Pi(u_i)\}_{i=1}^{\beta}, \{\mathsf{Query}^{(2)}_\Pi(z_i)\}_{i=1}^{\beta'}$ where each
$u_i \xleftarrow{\$} [R], z_i \xleftarrow{\$} [N]$ are independently uniform.

\textbf{Indistinguishability.}
The simulated view consists of $\beta+\beta'$ uniformly random PIR queries; the real view
consists of $\beta$ PIR queries for indices $r_1, \ldots, r_{\beta} \in [R], \zeta_1,\dots,\zeta_{\beta'}\in [N]$,
determined by the query keywords.
By Lemma~\ref{lem:composition} (applied with database size $R$ and $T = \beta$ queries; database size $N$ and $T=\beta'$ queries), the real and simulated views are computationally indistinguishable, completing the proof.
\end{proof}

\begin{theorem}[Security of \tree]
\label{thm:bm25tree}
Suppose $\Pi_{PIR}$ is a computationally private PIR scheme
(Definition~\ref{def:pir}).
Then Protocol~2 (\tree, Figure~\ref{fig:bm25-tree-protocol}) is secure against a semi-honest server
\end{theorem}

\begin{proof}
\tree proceeds in three sequential rounds; we analyze each in turn and then
apply Lemma~\ref{lem:composition}.

\textbf{Round 1 (Block-Max Tree Traversal).}
At level $\ell \in \{1, \ldots, \mathcal H\}$ of the $r$-ary tree with height
$\mathcal H = 1+\lceil \log_r (N/B) \rceil$, the client expands the beam frontier
$F_{\ell-1}$ of at most $W$ nodes by fetching the block-max scores of all
$|F_{\ell-1}| \cdot r \leq W \cdot r$ children via PIR.
Because $W$ and $r$ are public parameters, the number of PIR queries at each
level is at most $W \cdot r$ regardless of the query content $Q$.
The server therefore observes at most $\mathcal H \cdot W \cdot r$ PIR queries over a
database of $N_{\mathrm{nodes}}$ tree nodes, where $N_{\mathrm{nodes}}$ and
$\mathcal H \cdot L \cdot r$ are both fixed public quantities.

\textbf{Round 2 (Sub-Block Retrieval).}
The client fetches block-max scores for exactly $Z = W\cdot (B/s)$ sub-blocks, a count determined entirely by public parameters $(W, B, s)$.
The server observes exactly $Z$ PIR queries over the sub-block database.

\textbf{Round 3 (Dense Embedding Fetch).}
The client fetches exactly $k_s\cdot s$ document embeddings via PIR over the
database of $N$ documents. The count $k_s \cdot s$ is a fixed public parameter.

\textbf{Sequential Composition.}
The total number of PIR queries across all three rounds is
$T = \mathcal H \cdot W \cdot r + Z + k_s \cdot s$, which is a deterministic polynomial
function of the public parameters.
By Lemma~\ref{lem:composition} applied to this $T$-query sequence, the server's
view across all rounds is computationally indistinguishable from a simulation
issuing $T$ uniformly random PIR queries, completing the proof.
\end{proof}

\begin{remark}[Access-Pattern Privacy]
\label{rem:access}
A critical property enabling Theorems~\ref{thm:bin1},~\ref{thm:bin2},
and~\ref{thm:bm25tree} is that the \emph{number} of PIR queries per round is a
deterministic function of public parameters $(N, B, s, r, W, k_s, q^\star,\beta,\beta')$,
independent of the query content $Q$.
This eliminates access-pattern leakage: the server cannot distinguish any two
queries of equal length $|Q|=q^\star$.
By contrast, adaptive protocols such as PACMANN must pad short-circuit
executions with dummy queries to hide early convergence, introducing both
efficiency overhead and a potential timing side-channel if padding is imperfect.
\end{remark}

\begin{remark}[Correctness of Algorithms]
We follow the common practice of finding approximate nearest neighbor of modern framework like HNSW ~\citep{malkov2018hnsw} or FAISS~\citep{johnson2019faiss}: The retrieval correctness are not formally analyzed but rather showed by empirically evaluate it using benchmark dataset ~\citep{benchmark_ann}. Thus, we show the correctness experimentally in Section ~\ref{sec:evaluation} and omit the proof on theoretical correctness gap between \bin, \tree and the ideal top-$k$ nearest neighbors search.
\end{remark}

\section{PIR \tree Database via Cuckoo Hashing}
\label{app:cuckoo}

\textbf{Notation.} With the term $\mathsf{Tree}(\corpus)$ we denote the tree build upon corpus $\corpus$ while with the term $\mathcal{L}_\ell$ we denote the collection of superblocks of layer $\ell$ of the tree. 
We use $H=O(\log_r |\corpus|/|\mathcal{B}|)$ to denote the height of the tree. 
For ease of exposition, we may abuse the notation $\mathcal{B}$ for both the block and the super block in the tree.  
Let $\textsf{id}(\mathcal{B}_{\ell})$ denote the index of block $\mathcal{B}_{\ell}$ within level $\ell$. The client can compute this value locally, i.e., to fetch, say, the 4th node of level $\ell$, it simply sets $\textsf{id}(\mathcal{B}_{\ell}) = 4$.

\textbf{Storage Blowup.}
Recall the initial proposal from the main paper in which at each level $\ell$, every super block $\mathcal{B} \in \mathcal{L}_\ell$ is associated with an array of $|\mathcal{V}|$ integers storing the upper bounds.
Such an array is extremely sparse (most of its entries are $0$) so this representation wastes a substantial amount of storage.
Crucially, the cost is not confined to space, i.e., PIR performance scales with the size of the PIR database, which here is $O(|\mathcal{L}_\ell| \cdot |\mathcal{V}|)$ for level $\ell$, so the same redundancy inflates query time as well.
Summing over all levels, representing each level as a concatenation of $|\mathcal{L}_\ell|$ sparse arrays forces the server to store $O\!\left(\sum_{\ell} |\mathcal{L}_\ell| \cdot |\mathcal{V}|\right)$ integers, the vast majority of which are $0$.

\textbf{Proposed Optimization.} In order to reduce the PIR database storage, for each level $\ell$ on the tree, we present its corresponding PIR database as a \emph{dictionary} (specifically, a hash map)  of \textsf{key}-\textsf{value} pairs where the key is the string $\textsf{key}=w||\textsf{id}(\mathcal{B}_{\ell})$ and the value is $\textsf{value}=\sigma_w(\mathcal B_\ell)||w||\textsf{id}(\mathcal{B}_{\ell})$. 
The dictionary stores an entry only for those blocks whose score is non-zero, i.e., $\sigma_w(\mathcal{B}_\ell) > 0$, which is precisely what eliminates the sparsity of the previous representation.
This, however, means the client may issue a query on a combination of $w$ and $\textsf{id}(\mathcal{B}_{\ell}))$ that has no entry in the dictionary.
Because our dictionary always returns \emph{some} entry (it cannot signal that a key is absent) the client needs a way to tell a genuine answer from an unrelated entry it happened to collide with.
Repeating $w \,\|\, \textsf{id}(\mathcal{B}_{\ell})$ inside $\textsf{value}$ supplies exactly this check; the client compares the returned suffix against the key it queried, and interprets a mismatch as $\sigma_w(\mathcal{B}_\ell) = 0$. 
Concretely, on level $\ell$ of $\textsf{Tree}(\corpus)$, let $N_\ell$ be the number of dictionary entries such that $w\in\mathcal V, \mathcal B_\ell$ on level $\ell$ of the tree, and $\sigma_w(B_\ell) > 0$. 
Thus, we have $N_\ell \ll |\mathcal L_\ell|\cdot |\mathcal V|$. 
The hash map is implemented via a \emph{Cuckoo Hash Table}, with load factor $\alpha$, which has $O(1)$ search time. 
We use $m$ hash functions $h_1,\dots, h_m: \mathcal V\times \mathcal{L}_\ell \to [N_\ell/\alpha]$ to map dictionary entries to a table of $N_\ell/\alpha$ slots such that each slot contains at most one \textsf{key}-\textsf{value} pair, and the entry $\left(w||\textsf{id},\sigma_w(\mathcal B_\ell)||w||\textsf{id}(\mathcal{B}_{\ell})\right)$ is mapped to slot $h_i(w||\textsf{id})$ for one $i\in \{1,\dots,m\}$. 
For an empty row in Cuckoo table, the data is $(\bot, 0)$. 
At query time, suppose the client wishes to learn the score of a block $\mathcal{B}_c \in \mathcal{L}_\ell$ with respect to one of its own keywords $w_c$.
It forms the key $\textsf{key}_c = w_c \,\|\, \textsf{id}(\mathcal{B}_c)$, issues $m$ PIR queries for the candidate slots $h_1(\textsf{key}_c), \dots, h_m(\textsf{key}_c)$, and receives $m$ values.
Applying the check described above, the client compares each returned suffix against $\textsf{key}_c$ to verify if this word appears in the corresponding superblock.

\textbf{Complexity.} On layer $\ell$ of the tree, the new representation incurs $m\cdot w \cdot r\cdot |Q|$ queries on a database of size $N_\ell/\alpha \cdot 3$ ($N_\ell/\alpha$ rows, three fields per row). On comparison, the old, unoptimized version has $w\cdot r \cdot |Q|$ queries on a database of size $|\mathcal V|\cdot |\mathcal L_\ell|$ (i.e., $|\mathcal V| \cdot|\mathcal L_\ell|$ rows with one field per row). In total over $H$ levels, the storage size of the new representation is $\sum_{\ell=1}^{H} N_\ell/\alpha$ comparing to the old cost of $O(|\mathcal{V}|\cdot\frac{|\corpus|}{|\mathcal{B}|})$, where $\sum_{\ell=1}^{H} N_\ell \ll \sum_{\ell=1}^{H} |\mathcal L_\ell|\cdot|\mathcal V| = O(|\mathcal{V}|\cdot\frac{|\corpus|}{|\mathcal{B}|})$.

\section{Supplementary details}

\subsection{Full Testbed.}
\label{app:testbed}

We detail the extra details regarding our testbed that would be too verbose for the main body:
We also test LAN latency in Appendix~\ref{app:experiments}, which has 1000 megabits per second bandwidth and 5 millisecond round-trip-time.
Every configuration was given $120$\,GB of RAM, with the exception of $\bin$ configurations that: hold the embedding database in memory and have more than $1.5$K documents per row, those were run on a machine with $1$\,TB of RAM.
For SciFact's document/query embeddings we used $192$-dimension embeddings created from the \texttt{e5-base-v2}~\citep{wang2024textembeddingsweaklysupervisedcontrastive} embedding function. 
For MS MARCO's document/query embeddings  we used $192$-dimension embeddings created from the \texttt{msmarco-MiniLM-L-6-v3} (The same as PACMANN did, to keep the comparison fair), Sentence-BERT bi-encoder~\citep{reimers-gurevych-2019-sentence}, built on MiniLM~\citep{wang2020minilmdeepselfattentiondistillation} and fine-tuned on MS MARCO.
We ran our LLM/AI experiments on a single NVIDIA A100 (80GB VRAM) and fixed the RAGAS evaluation stage to the same $100$ queries per dataset.

\subsection{Hybrid Retrieval in Practice.}
\label{app:hybrid_retrieval}
Combining lexical and dense retrieval is a standard, first-class feature of the search engines on which RAG pipelines are built.
Elasticsearch exposes a reciprocal rank fusion (RRF) retriever~\citep{2009cormack} that merges the rankings of a BM25 query and a $k$-NN query into a single result list~\citep{elastic_rrf}.
OpenSearch introduced a dedicated hybrid query in version~2.11 that normalizes and combines keyword and neural scores inside the engine~\citep{opensearch_hybrid}.
Weaviate runs vector and BM25 search in parallel and fuses the two result sets with a configurable weighting~\citep{weaviate_hybrid}.
Azure AI Search likewise executes keyword and vector queries in parallel and merges them with RRF, and its documentation reports that, in most benchmark tests, hybrid queries with semantic reranking return the most relevant results~\citep{azure_hybrid}.
A Microsoft whitepaper, drawing on production RAG applications serving billions of queries per day, describes hybrid search with reranking as ``table stakes’' for RAG~\citep{azure_whitepaper}.
Model providers make the same recommendation: Anthropic’s Contextual Retrieval pipeline pairs embeddings with BM25, and adding the BM25 component reduced the top-20 retrieval failure rate from $3.7\%$ to $2.9\%$~\citep{anthropic2024contextual}.
These sources establish that hybrid retrieval is broadly supported and recommended in practice; they do not constitute a census of deployed systems, and many simple RAG pipelines remain dense-only.

\end{document}